\documentclass{article}

\usepackage{hyperref}
\usepackage[dvipsnames]{xcolor}
 \hypersetup{
    colorlinks=true,
    linkcolor=MidnightBlue,
    citecolor=Green
 }

\usepackage[preprint]{neurips_2026}

\usepackage[utf8]{inputenc} %
\usepackage[T1]{fontenc}    %
\usepackage{url}            %
\usepackage{booktabs}       %
\usepackage{amsfonts}       %
\usepackage{nicefrac}       %
\usepackage{microtype}      %

\usepackage{amsmath}
\usepackage{amssymb}
\usepackage{mathtools}
\usepackage{amsthm}

\usepackage{enumitem}

\usepackage[capitalize,noabbrev]{cleveref}

\usepackage{crossreftools}
\usepackage{xargs} 
\usepackage{xstring}
\usepackage{etoolbox}

\usepackage{etoc}
\usepackage[bibliography=common]{apxproof}
\usepackage{aliascnt}

\renewcommand*{\appendixprelim}{}

\newtheoremrep{thm}{Theorem}[section]
\newtheoremrep{cor}[thm]{Corollary}
\newtheoremrep{lemma}[thm]{Lemma}
\newtheoremrep{fact}[thm]{Fact}
\newtheoremrep{prop}[thm]{Proposition}

\newaliascnt{assumption}{thm}
\newtheorem{assumption}[assumption]{Assumption}
\aliascntresetthe{assumption}

\theoremstyle{remark}
\newtheorem{remark}[thm]{Remark}

\crefname{thm}{theorem}{theorems}
\crefname{assumption}{assumption}{assumptions}
\Crefname{assumption}{Assumption}{Assumptions}
\crefname{cor}{corollary}{corollaries}
\Crefname{cor}{Corollary}{Corollaries}
\crefname{prop}{proposition}{propositions}
\Crefname{prop}{Proposition}{Propositions}
\crefname{lemma}{lemma}{lemmas}
\Crefname{lemma}{Lemma}{Lemmas}
\crefname{fact}{fact}{facts}

\usepackage{mdframed}
\newmdtheoremenv{algo}{Algorithm}
\newmdtheoremenv{procedure}{Decision process}

\newlist{assnum}{enumerate}{1} %
\setlist[assnum]{label=(\roman*), ref=\theassumption(\roman*)}
\crefalias{assnumi}{assumption} 

\newlist{lemnum}{enumerate}{1} %
\setlist[lemnum]{label=(\roman*), ref=\thelemma(\roman*)}
\crefalias{lemnumi}{lemma} 

\newlist{thmnum}{enumerate}{1} %
\setlist[thmnum]{label=(\roman*), ref=\thethm(\roman*)}
\crefalias{thmnumi}{thm} 

\newlist{cornum}{enumerate}{1} %
\setlist[cornum]{label=(\roman*), ref=\thecor(\roman*)}
\crefalias{cornumi}{cor} 

\newlist{definitionnum}{enumerate}{1} %
\setlist[definitionnum]{label=(\roman*), ref=\thedefinition(\roman*)}
\crefalias{definitionnumi}{definition} 

\newlist{propnum}{enumerate}{1} %
\setlist[propnum]{label=(\roman*), ref=\theproposition(\roman*)}
\crefalias{propnumi}{prop}

\newlist{examplenum}{enumerate}{1} %
\setlist[examplenum]{label=(\roman*), ref=\theexample(\roman*)}
\crefalias{examplenumi}{example} 

\newcommand\numberthis{\addtocounter{equation}{1}\tag{\theequation}}		%

\usepackage{dsfont}
\usepackage{nicefrac}

\DeclareMathOperator*{\argmax}{arg\,max}
\DeclareMathOperator*{\argmin}{arg\,min}

\newcommand{\R}{\mathbb{R}}

\newcommand{\lmo}{\operatorname{lmo}}
\newcommand{\sign}{\operatorname{sign}}

\newcommand{\norm}[1]{\left\Vert{#1}\right\Vert}

\newcommand{\RMS}{\mathrm{RMS}}

\usepackage{tcolorbox}
\newtcolorbox{defbox}{colback=black!5!white,colframe=black!75!black}
\newtcolorbox{asmbox}{colback=black!5!white,colframe=black!75!black}
\newtcolorbox{thmbox}{colback=red!5!white,colframe=red!75!black}

\usepackage{braket}

\newcommandx{\QC}[2][1={},2={}]{\ifstrempty{#1}{Q#2}{Q_{#1}\ifstrempty{#2}{}{(#2)}}}
\newcommandx{\PC}[2][1={},2={}]{\ifstrempty{#1}{P#2}{P_{#1}\ifstrempty{#2}{}{(#2)}}}
\newcommandx{\HC}[2][1={},2={}]{\ifstrempty{#1}{H#2}{H_{#1}\ifstrempty{#2}{}{(#2)}}}
\newcommandx{\MC}[2][1={},2={}]{\ifstrempty{#1}{M#2}{M_{#1}\ifstrempty{#2}{}{(#2)}}}

\newcommandx{\EF}[2][1={k},2={}]{\mathbb E\ifstrempty{#1}{}{_{#1}}#2}

\usepackage{algorithm}
\usepackage[noend]{algpseudocode}
\newcommand{\algfont}[1]{\textbf{#1}}

\usepackage{adjustbox}

\usepackage{array,multirow,hhline,graphicx}
\usepackage{wrapfig}
\usepackage{colortbl}

\usepackage{pifont}
\newcommand{\xmark}{\ding{55}}

\usepackage[para,online,flushleft]{threeparttable}

\title{Optimistic Dual Averaging Unifies Modern Optimizers}

\author{%
  Thomas Pethick \\
  Independent Researcher \\
  \texttt{tmpethick@gmail.com}
  \And
  Wanyun Xie \\
  EPFL (LIONS)\\
  \texttt{wanyun.xie@epfl.ch}
  \AND
  Roman Machacek \\
  University of Bern\\
  \texttt{roman.machacek@unibe.ch}
  \And
  Volkan Cevher \\
  EPFL (LIONS) \\
  \texttt{volkan.cevher@epfl.ch}
}

\begin{document}

\maketitle 

\begin{abstract}
We introduce SODA, a generalization of Optimistic Dual Averaging, which provides a common perspective on state-of-the-art optimizers like Muon, Lion, AdEMAMix and NAdam, showing that they can all be viewed as optimistic instances of this framework.
Based on this framing, we propose a practical SODA wrapper for any base optimizer that eliminates weight decay tuning through a theoretically-grounded $1/k$ decay schedule. 
Empirical results across various scales and training horizons show that SODA consistently improves performance without any additional hyperparameter tuning.

\end{abstract}

\etocdepthtag.toc{mtchapter}
\etocsettagdepth{mtchapter}{subsection}
\etocsettagdepth{mtappendix}{none}

\section{Introduction}

Deep learning optimization has developed along two largely complementary axes.
The first axis is \emph{geometry and adaptation}: designing updates whose
implicit norm, preconditioner, or constraint set reflects the structure of
modern models, potentially improving scaling in high dimensions.
Adam \citep{kingma2014adam} is a central example, coupling gradient averaging
with an $\ell_\infty$-type geometry through elementwise normalization.
In parallel, \emph{spectral} geometry and stochastic dualization were developed
in Stochastic Spectral Descent (SSD)
\citep{carlson2015stochastic,carlson2016stochastic}, and are now resurfacing
in recent methods such as Muon and Scion
\citep{jordan2024muon,pethick2025trainingdeeplearningmodels}.

The second axis is \emph{composition of ingredients}: how gradient feedback and iterates are combined through momentum, averaging, and schedules to obtain a stable and performant training recipe.
Algorithms such as Lion \citep{chen2023symbolic} and practical
recipes such as the schedule-free wrapper \citep{defazio2024road} highlight that the assembly of these ingredients can be as important
as the ingredients themselves.
A striking pattern over the last decade is that many seemingly disparate new optimizers can be explained as different points in this 2D design space.

A useful lens for both axes is the classical dualization framework, where the next iterate is obtained by minimizing a linear surrogate regularized by a geometry-inducing term:
\begin{equation}
\label{eq:gen_form}
\textstyle x^{k+1} \in \argmin_{x \in \mathcal{X}} \gamma_k \langle d^{k}, x \rangle + h_k(x).
\end{equation}
Here, $\gamma_k>0$ is a stepsize schedule, $d^k$ is a (possibly averaged) gradient feedback, and $h_k$ is a (possibly time-varying) regularizer/mirror map that determines the geometry.

For instance, choosing $h_k(x)=\tfrac{1}{2}\|x-x^k\|_2^2$ recovers gradient descent, while non-Euclidean choices of $h_k$ yield normalized and geometry-aware updates, that can depend more favorably on the dimensionality of the problem.
In parallel, \emph{how} we form $d^k$ (momentum, averaging, extrapolation) and \emph{how} we average iterates (primal averaging / schedule-free) 
have a strong effect on the properties of the resulting method, including noise robustness, acceleration, and anytime guarantees.

One major challenge is how to set hyperparameters of these optimizers,
which becomes particularly important as model size becomes prohibitively expensive to tune and we instead seek predictable scaling rules.

One particularly challenging parameter is weight decay.
In the context of multi-epoch training, weight decay has a precise characterization, since it constrains the norm of the iterates, thus acting as a regularizer \citep{xie2024implicit,pethick2025trainingdeeplearningmodels}.
However, even in single-epoch training, where overfitting is not a concern, weight decay can surprisingly still be beneficial.

Very recently a scaling rule was developed choosing weight decay as $1/d$ with model dimension $d$ \citep{xiao2024rethinking,qiu2025hyperparameter}.
However, since many large-scale experiments follow compute-optimal training regimes where $d$ and the training horizon $n$ are coupled, it is difficult to disentangle whether the effective dependence is on $d$ or on time.
In addition, weight decay typically needs to be tuned on a per-optimizer basis, as observed for Lion, which requires a significantly larger weight decay than AdamW \citep{chen2023symbolic}. 
It raises the following question:

\begin{center}
\emph{Is it possible to ground these two axes of recent progress in classical methods \\ and provide guidance on hyperparameters?}
\end{center}

\begin{table*}[!t]
\centering
\caption{Instances of \ref{eq:SODA} for $\bar\lambda_k=0$ (\textit{c.f.} \Cref{sec:method}).}
\label{tbl:soda}
\resizebox{\textwidth}{!}{
\scalebox{1.0}{
\begin{threeparttable}
\begin{tabular}{|l|c|c|c|c|}
\hline
\textbf{Method} & \textbf{Reference} & \textbf{Grad. avg ($\alpha_k<1$)} & \textbf{Optimism ($\bar\alpha_k>0$)} & \textbf{Geometry ($h_k$)} \\
\hline
\hline

{Adam} &  \cite{kingma2014adam}
& $\checkmark$
& \xmark
& Smoothed $\ell_\infty$  \\
\hline
{NAdam}\tnote{1} &  \cite{dozat2016incorporating}
& $\checkmark$
& $\checkmark$
& Smoothed $\ell_\infty$  \\
\hline
\hline

Stoch. $\ell_\infty$ &   \cite{carlson2015stochastic,carlson2016stochastic}
& \xmark
& \xmark
& $\ell_\infty$ \\
\hline
Signum  &   \cite{bernstein2018signsgd}
& $\checkmark$
& \xmark
& $\ell_\infty$ \\
\hline
Lion  &   \cite{chen2023symbolic}
& $\checkmark$
& $\checkmark$
& $\ell_\infty$ \\
\hline
\hline

{Shampoo} &  {\cite{Gupta2018ShampooPS}}
& \xmark
& \xmark
&  Smoothed Spectral  \\
\hline
\hline
\shortstack{SSD} &   {\cite{carlson2015stochastic,carlson2016stochastic}}
& \xmark
& \xmark
& Spectral  \\
\hline
Scion {\footnotesize (Spectral)} &  { \cite{pethick2025trainingdeeplearningmodels}}
& $\checkmark$
& \xmark
& Spectral \\
\hline
Muon &  { \cite{jordan2024muon}}
& $\checkmark$
& $\checkmark$
& Spectral \\
\hline
\hline
Multi-Norm &  { \cite{scetbon2025gradient}}
& \xmark
& \xmark
& Doubly Stochastic \\
\hline
\end{tabular}
\begin{tablenotes}
\item[1] Rediscovered as Simplified-AdEMAMix \citep{morwani2025connections}.
\end{tablenotes}
\end{threeparttable}
}
}
\end{table*}

In this paper, we contend that the two axes of algorithmic development are not separate threads.
To this end, we introduce \ref{eq:SODA}, a generalization of Optimistic Dual Averaging \citep{rakhlin2013online} that explicitly couples \emph{dual processing} (gradient averaging + optimism) and \emph{primal processing} (iterate averaging + primal extrapolation). %
The resulting framework provides a single template that recovers several widely used optimizers as special cases and makes it straightforward to inject non-Euclidean geometry, such as $\ell_\infty$ or spectral geometry, into modern training recipes.

In particular, when one of the averaging parameters is simplified (our ``modernized'' regime), SODA yields a practical wrapper around any base optimizer that eliminates weight decay tuning via a theoretically grounded $1/k$ decay. %

\vspace{-3mm}
\paragraph{Contributions.}
Our contributions are as follows:
\vspace{-3mm}
\begin{enumerate}[itemsep=0pt,label=($\roman*$),leftmargin=30pt]
    \item \textit{Unification via optimism and dualization:}
    We show that \ref{eq:SODA} provides a unified perspective on several state-of-the-art optimizers.
    Most notably, it captures Muon, Lion-$\mathcal{K}$, and NAdam as \emph{optimistic} instances within this broad framework (\textit{cf.} \Cref{tbl:soda}).
    \item \textit{Theoretical guarantees:}
    The \ref{eq:SODA} framing allows us to theoretically derive hyperparameter choices that works remarkably well in practice. %
    Our formulation provides a new perspective on weight decay, not only as a regularizer, but as a form of primal averaging. 
    This results in a new update rule that anchors each update at the initialization and induces a $1/k$ weight decay schedule. 
    \item \textit{A simple wrapper:}
    We propose a practical wrapper around any base optimizer (e.g., Adam, Lion, Muon, Scion) that introduces no new hyperparameters and removes the need to tune a weight decay parameter. %
    \item \textit{Empirical results without additional tuning:}
    We demonstrate consistent performance improvements when wrapping Adam, Muon, and Scion across various model sizes and training horizons, even outperforming baselines with tuned weight decay.
    Our proposed instantiation of our framework, SODA$^\dag$, removes the need to tune weight decay while leading to better performance at scale.
\end{enumerate}

\paragraph{Limitations}
This work focuses on the single-epoch setting. 
For multi-epoch training it might be necessary to have a hyperparameter controlling the regularization strength, due to overfitting.

\begin{toappendix}
\section{Preliminaries}\label{app:prelim}
\end{toappendix}
\section{Preliminaries}\label{sec:prelim}
\nosectionappendix
We are interested in the following minimization problem
\begin{equation}\label{eq:stochastic_problem}
\min_{x \in \mathcal X} f(x) := \mathbb E_\xi [f(x, \xi)],
\end{equation}
which covers a host of machine learning problems, where $\xi$ represents a data sample, $x$ is the parameter of the model being optimized, and $f(x, \xi)$ is a loss function.    

\paragraph{Dual averaging and Fenchel conjugate}
The Dual Averaging framework \citep{nesterov2009primal} and its variants typically rely on the Fenchel conjugate to map gradient information back to the primal space. The Fenchel conjugate of a function $h: \mathcal X \to \R \cup \{\infty\}$ is defined as:
\begin{equation*}
h^*(d) = \sup_{x \in \mathcal X} \left\{ \langle d, x \rangle - h(x) \right\},
\end{equation*}
where $d \in \mathcal X^*$. 
The subdifferential $\partial h^*$ is equivalent to the set of maximizers of this conjugate operation:
\begin{equation*}
\partial h^*(d) = \operatorname*{argmax}_{x \in \mathcal X} \left\{ \langle d, x \rangle - h(x) \right\}.
\end{equation*}
This identity is a direct consequence of the Fenchel-Young inequality \citep{bauschke2012fenchel}. 
The general optimization template \eqref{eq:gen_form} can then be written compactly as:
\begin{equation}\label{eq:dual_update_template}
x^{k+1} \in \partial h_k^*(-\gamma_k d^k),
\end{equation}
where $d^k$ is an average (or momentum) of past stochastic gradients, $\gamma_k$ is a stepsize schedule, and $h_k$ is a sequence of geometry defining regularizers (or mirror maps).

\vspace{-3mm}
\paragraph{Standard regularizers}
Different choices of the regularizer $h(x)$ recover well-known optimization steps. We highlight three instances that are particularly relevant for our work, especially when the constraint set is the norm ball $\mathcal{D}=\{x \in \mathcal X \mid \|x\|\leq 1\}$ for some arbitrary norm $\|\cdot\|$:\vspace{-2mm}
\begin{enumerate}[itemsep=0pt, label=($\roman*$)]
    \item \textit{Unconstrained:} Let $h(x) = \frac{1}{2}\|x\|^2$. Then, the $\partial h^*$ is the sharp operator \citep{nesterov2012efficiency}:
    \begin{equation*}
    \begin{split}
        \partial h^*(d) = [d]^\sharp := \argmax_{x} \left\{ \braket{d,x} - \tfrac{1}{2}\|x\|^2 \right\}.
    \end{split}
    \end{equation*}
    \vspace{-5mm}
    \item \textit{Linear minimization oracle ($\lmo$):} Let $h(x) = \iota_{\mathcal D}(x)$ be the indicator function of a set $\mathcal D$ (i.e., $0$ if $x \in \mathcal D$ and $\infty$ otherwise). The $\partial h^*$ is then the $\lmo$:
    \begin{equation*}
        \begin{split}
            \partial h^*(-d) = \lmo_{\mathcal D}(d) := \argmin_{x \in \mathcal D} \braket{d,x}.
        \end{split}
    \end{equation*}
    This is the core operation in Frank-Wolfe (conditional gradient) methods \citep{frank1956algorithm,ken-fw,jaggi2013revisiting}, which are projection-free but typically require $f$ to be smooth.
    \item \textit{Clipping:} By combining the above, we can handle a non-smooth $f$ via ``clipping.'' Let $h(x) = \frac{1}{2}\|x\|^2 + \iota_{\mathcal D}(x)$. Then, we have 
    \begin{equation*}
    \begin{split}
        \partial h^*(-d) = \operatorname{clip}_{\mathcal D}(d) := \argmin_{x \in \mathcal D} \{ \braket{d,x} + \tfrac{1}{2}\|x\|^2 \}.
    \end{split}
    \end{equation*}
    This operation clips the gradient step into the feasibility set $\mathcal D$, essentially performing a projection of the negative gradient.
    This was used in \citet{pethick2025generalized,crawshaw2025exploration} for training neural networks.
\end{enumerate}
\vspace{-1mm}

These three choices of $\partial h^*$ can all be related through the $\lmo$, since $[d]^\#=-\|d\|_*\lmo_{\mathcal{D}}(d)$ and $\operatorname{clip}_{\mathcal D}(d) = \min \{1, \|d\|_*\} \lmo_{\mathcal D}(d)$ (see \Cref{lem:equiv}).

\paragraph{Norm choices}
Important special cases arise by taking $h=\iota_{\mathcal D}$ with $\mathcal D$ the unit ball of a norm. For the $\ell_\infty$ ball, $\lmo_{\mathcal D}(d)=-\operatorname{sign}(d)$ as used in SignSGD and Lion \citep{bernstein2018signsgd,chen2023symbolic}. For the spectral norm ball, $\lmo_{\mathcal D}(G)=-\operatorname{msign}(G)$ with $\operatorname{msign}(G):=UV^\top$, where $U\Sigma V^\top$ is the singular value decomposition of $G$ as used in \citet{carlson2016stochastic,jordan2024muon,pethick2025trainingdeeplearningmodels}.

\begin{toappendix}
\begin{lemma}\label{lem:equiv}
Let \(\|\cdot\|\) be a norm with dual norm \(\|\cdot\|_*\), and let $\mathcal D=\{x:\|x\|\le 1\}$. Then
\[
s^\sharp=-\|s\|_*\,\lmo(s).
\]
\end{lemma}

\begin{proof}
Let
\[
u\in\argmax_{\|v\|\le 1}\langle s,v\rangle,
\]
so that \(\langle s,u\rangle=\|s\|_*\). Then
\[
\lmo(s)\in\argmin_{\|x\|\le1}\langle s,x\rangle=- u.
\]
Also, writing \(x=t v\) with \(t\ge0\) and \(\|v\|=1\),
\[
s^\sharp\in\argmax_{t\ge0,\ \|v\|=1}\Bigl\{t\langle s,v\rangle-\tfrac12 t^2\Bigr\}.
\]
For fixed \(t\), the maximization over \(v\) is independent of \(t\), hence uses the same \(u\). Therefore
\[
s^\sharp\in\argmax_{t\ge0}\Bigl\{t\|s\|_*-\tfrac12 t^2\Bigr\}=\|s\|_*u.
\]
Combining the two identities gives
\[
u=- \lmo(s),
\]
and thus
\[
s^\sharp=- \|s\|_*\,\lmo(s).\qedhere
\]
\end{proof}
\end{toappendix}

\section{Method}\label{sec:method}
\nosectionappendix

We propose the following algorithm which generalizes Optimistic Dual Averaging (ODA) by introducing a primal extrapolation sequence ($y^{k}$) from \citet{tseng2008accelerated,lan2012optimal,defazio2024road}:
\begin{equation}\label{eq:SODA}
\tag{SODA}
\begin{split}
m^{k+1} &= (1-\alpha_k) m^k + \alpha_k \nabla f(y^k,\xi_k) \\
\bar{m}^{k+1} &= (1-\bar{\alpha}_k) m^{k+1} + \bar{\alpha}_k \nabla f(y^k, \xi_k) \quad \textcolor{black}{\text{(optimism)}}\\
z^{k+1} & \in \partial h_k^*(-\gamma_k\bar{m}^{k+1}) = \argmin_{x \in \mathcal X} \gamma_k\braket{\bar{m}^{k+1},x} + h_k(x) \\ 
x^{k+1} &= (1-\lambda_k)x^k + \lambda_k z^{k+1} \\
y^{k+1} &= (1 - \bar\lambda_k) x^{k+1}  + \bar\lambda_k z^{k+1}  \qquad\qquad \textcolor{black}{\text{(primal extrapolation)}}
\end{split}
\end{equation}
with $\alpha_k, \bar\alpha_k, \lambda_k, \bar\lambda_k \in [0,1]$ and $\gamma_k>0$.
We initialize $m^0=0$ and $x^0=y^0=z^0\in\partial h_0^*(0)$.
Observe the elegant symmetry between how the dual (gradients) and primal (iterates) are being processed.
Our analysis in \Cref{sec:analysis} builds on the \underline{S}chedule-Free framework of \citet{defazio2024road} and \underline{ODA} \citet{rakhlin2013online}, so we refer to the method as SODA to pay homage.

Note that typically in machine learning libraries such as PyTorch the momentum parameters are instead defined as $\beta_k=1-\alpha_k$, $\bar\beta_k=1-\bar\alpha_k$, $\tau_k=1-\lambda_k$, and $\bar\tau_k=1-\bar\lambda_k$. 

\subsection{Special Cases}

There are two important extremes depending on the choice of the primal extrapolation parameter $\bar\lambda_k$.

\paragraph{Optimistic Dual Averaging} 
For $\bar\lambda_k=1$:
\begin{equation}\label{eq:ODA}
\tag{ODA}
\begin{split}
m^{k+1} &= (1-\alpha_k) m^k + \alpha_k \nabla f(z^k,\xi_k) \\
\bar{m}^{k+1} &= (1-\bar{\alpha}_k) m^{k+1} + \bar{\alpha}_k \nabla f(z^k, \xi_k) \\
z^{k+1} & \in \partial h^*(-\gamma_k\bar{m}^{k+1}) \\ 
x^{k+1} &= (1-\lambda_k)x^k + \lambda_k z^{k+1} \\
\end{split}
\end{equation}
This is an optimistic version of the celebrated Dual Averaging scheme \citep{nesterov2005smooth}, also known as optimistic follow-the-regularized-leader (FTRL) in online learning \citep{rakhlin2013online}.
Notice that the output of the algorithm ($x^{k}$) can be different from where the gradient is evaluated ($z^k$).
\looseness=-1

\paragraph{Modernized Optimistic Dual Averaging} 
For $\bar\lambda_k=0$:
\begin{equation}\label{eq:MODA}
\tag{MODA}
\begin{split}
m^{k+1} &= (1-\alpha_k) m^k + \alpha_k \nabla f(x^k,\xi_k) \\
\bar{m}^{k+1} &= (1-\bar{\alpha}_k) m^{k+1} + \bar{\alpha}_k \nabla f(x^k, \xi_k) \\
x^{k+1} &= (1-\lambda_k)x^k + \lambda_k \partial h_k^*(-\gamma_k\bar{m}^{k+1})
\end{split}
\end{equation}
This can be seen as an optimistic version of Double Averaging of \citet{nesterov2015quasi}.
We refer to this as ``modernized'' following \citet{jelassi2020dual}, which applied Double Averaging to deep learning.
The output of the algorithm ($x^{k}$) is the same as where the gradient is evaluated ($x^k$).

This recovers Stochastic Frank-Wolfe and Scion \citep{mokhtari2020stochastic,pethick2025trainingdeeplearningmodels} with $\bar{\alpha}_k=0$ and 
$h_k(x) = \iota_{\mathcal D}(x)$, since then
$\partial h_k^*(-\gamma_k d) =  \rho \operatorname{lmo}(\gamma_k d) = \rho \operatorname{lmo}(d)$ for some constrained radius $\rho>0$ where the first equality follows from scale invariance of the $\lmo$.

More importantly, for $\bar{\alpha}_k\neq 0$, \ref{eq:MODA} captures NAdam \cite{dozat2016incorporating}, Lion \citep{chen2023symbolic} and Muon \citep{jordan2024muon} with weight decay through the choice of geometry $h_k$:
\begin{itemize}
    \item Lion-$\mathcal{K}$: choose $h_k = \mathcal{K}^*$, so that $\partial h_k^*=\partial \mathcal{K}$.
    For Lion specifically, $\partial h_k^*(-u)=-\operatorname{sign}(u)$.
    \item Muon: choose a spectral mirror map $h_k$ such that $\partial h_k^*(-u)=-\operatorname{msign}(u)$. %
    The so-called Nesterov momentum corresponds to choosing the optimistic parameter as $\bar\alpha_k = \alpha_k$. 
    \item NAdam: choose $h_k(x)=\tfrac{1}{2}\langle x,\operatorname{Diag}(\sqrt{v^k}+\varepsilon)x\rangle$, where
    $v^{k} = \tau v^{k-1} + (1-\tau)\nabla f(x^k, \xi_k)\odot \nabla f(x^k, \xi_k)$.
    Then
    $\nabla h_k^*(-\gamma_k \bar m^{k+1})= -\gamma_k \bar m^{k+1} \oslash (\sqrt{v^k}+\varepsilon)$.
\end{itemize}
In this light, all of the above methods can be interpreted as optimistic versions of Dual Averaging.
We note that NAdam was also rediscovered as a simplification of AdEMAMix \citep{ademamix2024} named Simplified-AdEMAMix from \citep{morwani2025connections}.
See \Cref{tbl:soda} for an overview.

Notice that $\lambda_k$ plays the role of the stepsize in this case.
This is in contrast with the following where we will instead let $\lambda_k=1/(k+2)$ and let $\eta_k$ be the stepsize of the base optimizer (typically using a linear or cosine schedule).

\begin{algorithm}[t]
\caption{SODA Wrapper}
\label{alg:SODA}
\textbf{Input:} Horizon $n$, initialization $z^0=x^0 \in \mathcal X$
\begin{algorithmic}[1]
    \For{$k = 0, \dots, n-1$}
        \State Sample $\xi_k \sim \mathcal P$ and compute the gradient $g^k = \nabla f(x^k,\xi_k)$
        \State $z^{k+1} = z^0 +(k+2)\operatorname{BaseUpdate}(g^k)$
        \State $x^{k+1} = (1-\tfrac{1}{k+2})x^k + \tfrac{1}{k+2}z^{k+1}$
    \EndFor
    \item[\algfont{Return}] $x^{n}$
\end{algorithmic}
{\footnotesize $\operatorname{BaseUpdate}$ refers to the update delta of the base optimizer \emph{without weight decay}, i.e., $u^{k+1} = u^k + \operatorname{BaseUpdate}(g^k)$.}
\end{algorithm}

\paragraph{SODA Wrapper}
In \Cref{alg:SODA}, we provide a particularly practical instantiation of \ref{eq:SODA}, that wraps an existing base optimizer, based on hyperparameter choices from \Cref{cor:soda-nonacc-general}.
Following the theory, this approach uses $\bar\lambda_k = 0$ (aka \ref{eq:MODA}), uniform iterate averaging ($\lambda_k = \nicefrac{1}{k+2}$), a stepsize $\gamma_k = (k+2) \eta_k$ and defines the regularizer relative to a reference iterate $z^0$ as $h_k(x) = \psi_k(x-z^0)$ with $\psi_k(x):=\tfrac{1}{2}\|x\|^2_{H_k}$ where $H_k$ is a positive-definite matrix.
Since $\partial\psi_k^*$ is positively homogeneous, we have
$\partial h_k^*(-\gamma_k\bar m^{k+1})
= z^0+\partial\psi_k^*(-\gamma_k\bar m^{k+1})
= z^0+\gamma_k\partial\psi_k^*(-\bar m^{k+1})$.

Any optimizer based on (optimistic) gradient momentum and dualization, such as Adam, NAdam, Scion, Muon, Signum, or Lion, can serve as the base optimizer. %
Notice that the $\operatorname{BaseUpdate}$ involves the stepsize $\eta_k$, while the wrapper introduces no new hyperparameters.

The resulting \Cref{alg:SODA} has a clear interpretation when simplifying the expression:
\begin{equation}\label{eq:SODA:wrapper:simplified}
x^{k+1} = 
    \underbrace{\tfrac{1}{k+2}z^0}_{\text{centering}}
    + \underbrace{(1-\tfrac{1}{k+2})x^k}_{\text{scheduled weight decay}}
    + \underbrace{\operatorname{BaseUpdate}(g^k)}_{\text{includes a stepsize schedule}}
\end{equation}
Let us compare \eqref{eq:SODA:wrapper:simplified} with the original (\emph{independent}) weight decay \cite{hanson1988comparing} formulation
\begin{equation}\label{eq:wd}
x^{k+1} = (1-\mu_k)x^k - \gamma_k \nabla f(x^k,\xi_k)
\end{equation}
where $\mu_k \geq 0$. 
First the iterates in \eqref{eq:SODA:wrapper:simplified} are anchored in the initialization $z^0$.
Second, the weight decay is decaying as $\mu_k = \nicefrac{1}{k+2}$ and is \emph{not} multiplied by the stepsize schedule as otherwise standard in e.g., decoupled weight decay \citep{loshchilov2017decoupled}, whereas the base optimizer is multiplied with the stepsize schedule (absorbed into $\operatorname{BaseUpdate}$).

The weight decay plays a different role in \ref{eq:SODA} than typically. 
It does not act as a regularizer, but is rather related to the averaging of iterates, which improves optimization.
Remarkably, as we shall see in \Cref{sec:experiments}, the SODA wrapper (\Cref{alg:SODA}) consistently improves upon the wrapped base optimizer, even outperforming the baseline with tuned weight decay while the SODA wrapper requires no such tuning.

\begin{remark}
\Cref{alg:SODA} appears to require storing the initialization $z^0$ throughout training.
Since $z^0$ is typically generated from a fixed seed, it can instead be reconstructed on demand at each iteration, trading a small amount of computation for a method with no additional memory.
\end{remark}

\subsection{Related Work}

We organize related work through the lens of \ref{eq:SODA}, which separates
($i$) dual averaging and optimism,
($ii$) primal averaging and primal extrapolation, and
($iii$) dualization through geometry (mirror maps).
This decomposition clarifies how several modern optimizers arise as special cases.

\paragraph{Gradient averaging.}
Several works for deep learning have highlighted the benefit of maintaining multiple dual sequences.
AdEMAMix \citep{ademamix2024} showed that keeping an additional dual averaging sequence on top of Adam-style momentum can significantly improve performance.
Simplified-AdEMAMix \citep{morwani2025connections} later demonstrated that explicitly adding back the current gradient is sufficient, yielding an update equivalent to NAdam \citep{dozat2016incorporating}.
Similarly, Muon \citep{jordan2024muon} uses the PyTorch implementation of Nesterov momentum, which can be captured by \ref{eq:SODA} with $\bar\alpha_k=\alpha_k$.
The sign-based optimizer Lion \citep{chen2023lion} also moves beyond simple averaging.
From the perspective of \ref{eq:SODA}, these gradient averaging schemes correspond precisely to forming the optimistic dual variable $\bar m^{k+1}$.

\paragraph{Iterate averaging.}
A complementary line of work in the deep learning literature focuses on averaging and extrapolation in the \emph{primal} space.
Lookahead \citep{zhangLookaheadOptimizerSteps2019} performs iterate averaging by maintaining a slow sequence ($x^k$) and generating extrapolated points ($z^{k+1}$) via an inner optimization loop, corresponding to the special case $\bar\lambda=0$.
DiLoCo \citep{douillard2023diloco} replaces simple averaging in the outer optimizer with PyTorch-style Nesterov momentum, demonstrating improvements in distributed training and, subsequently, even in single-node settings \citep{kallusky2025snoo}.
\looseness=-1

Generalized Primal Averaging (GPA) \citep{defazio2025smoothing} further simplifies DiLoCo by showing that an explicit inner loop is unnecessary: it suffices to average the iterates ($x^{k+1}$) and extrapolate ($y^{k+1}$).
Specifically, GPA can be written as
\begin{equation*}
\tag{GPA}
\begin{split}
z^{k+1} &= (1-\eta_k \mu)z^k + \eta_k\operatorname{BaseUpdate}(\nabla f(y^k, \xi_k)), \\
x^{k+1} &= (1-\lambda)x^k + \lambda z^{k+1}, \\
y^{k+1} &= (1-\bar\lambda)x^{k+1} + \bar\lambda z^{k+1}.
\end{split}
\end{equation*}
where $\mu \geq 0$ and $\lambda,\bar\lambda \in [0,1]$.
GPA uses a stepsize schedule for the base optimizer's $\eta_k$ and fixes the iterate-averaging parameters $(\lambda,\bar\lambda)$.
From the SODA viewpoint, GPA corresponds to constant primal averaging parameters $(\lambda,\bar\lambda)$.
Typical choices are $1-\bar\lambda=0.9$ and $1-\lambda=(1-\bar\lambda)^{1/H}\approx 0.9967$ with $H=32$, implying $\lambda<\bar\lambda$.
This contrasts with the smooth SODA regime, which requires the sufficient condition $\bar\lambda_k\le \lambda_k/10$ (\Cref{cor:soda-nonacc-general,cor:soda-acc}).
Moreover, whereas GPA relies on a base optimizer \emph{with} a tuned weight decay parameter $\mu$, the SODA wrapper (\Cref{alg:SODA}) wraps a base optimizer \emph{without} weight decay and eliminates the need for any additional hyperparameters.

Weight decay can also be viewed as a form of primal averaging.
\Citet{xiao2024rethinking} observed a $1/d$ scaling with model size $d$ and used it for hyperparameter transfer, while \citet{qiu2025hyperparameter} adopted this rule for spectral methods such as Muon.
In contrast, we separate model size and horizon effects: if the horizon is fixed and only model size varies, our framework does not predict a necessary $1/d$ scaling. %
Concurrently, \citet{ferbach2026logarithmic} considered the time-decaying rule $\lambda_k=\lambda\eta_k/k$.
In contrast, the SODA wrapper uses the parameter-free schedule $\lambda_k=1/(k+2)$.
Our work focuses on removing hyperparameters, while \citet{ferbach2026logarithmic} introduces new hyperparameters.

\paragraph{Acceleration and universality.}
Primal extrapolation, corresponding to $\bar\lambda_k>0$ in \ref{eq:SODA}, originates in accelerated gradient and proximal-gradient methods \citep{tseng2008accelerated,lan2012optimal}. 
Querying gradients at the averaged point, corresponding to $\bar\lambda_k=0$, was later used concurrently by \citet{cutkosky2019anytime,kavis2019unixgrad} to obtain universal methods, i.e., a single algorithm simultaneously attaining both the optimal smooth stochastic convex rate $O(L/n^2+\sigma/\sqrt n)$ and the nonsmooth rate $O(1/\sqrt n)$.
\Citet{joulani2020simpler} combined this averaging mechanism with adaptive Optimistic Dual Averaging/FTRL \citep{rakhlin2013online,mohri2016accelerating}.
\Citet{defazio2024road} later extended this analysis to allow for the larger range $\bar\lambda_k\le \lambda_k/10$ and used it to develop a optimization wrapper for deep learning with unknown training horizon.

\paragraph{Dualization and geometry.}
The choice of mirror map $h$ determines the geometry of the update and plays a central role in modern deep learning.
Recent work has emphasized that many deep learning optimizers are best understood through their induced norm \citep{bernstein2024old}.
Elementwise sign methods are the simplest example: the $\ell_\infty$ geometry gives rise to SignSGD \citep{bernstein2018signsgd} and Lion \citep{chen2023symbolic}, and has been used to partially explain the effectiveness of the popular Adam optimizer \citep{kunstner2023noise}.
For matrix parameters, spectral descent methods \citep{carlson2015stochastic,carlson2016stochastic,carlson2015preconditioned} and their modern variants, such as Muon \citep{jordan2024muon} and Scion \citep{pethick2025trainingdeeplearningmodels}, arise from spectral geometries.
Beyond single-norm geometries, multi-norm constructions enforcing both row- and column-normalization (doubly stochastic structure) have also been explored \citep{scetbon2025gradient,xie2025mhc}.
Whenever the mirror map admits a tractable Fenchel conjugate, such geometries naturally fit within the SODA framework.
\looseness=-1

\section{Analysis}\label{sec:analysis}
We now derive convergence guarantees for \ref{eq:SODA}, in order to set the hyperparameters. The proof is based on an online regret argument, so we
first let $g^k$ denote an arbitrary gradient-feedback sequence. In the
stochastic optimization setting of \ref{eq:SODA}, we take
$g^k := \nabla f(y^k,\xi_k)$.
We use the following assumptions, which are standard except for \Cref{ass:noise}.
\begin{assumption}[Convex]
\label{ass:sample-convex}
For every sample $\xi$, the function $f(\cdot,\xi)$ is convex.
\end{assumption}

\begin{assumption}[$L$-smooth]
\label{ass:smooth}
The function $f$ is $L$-smooth with respect to $\norm{\cdot}$.
\end{assumption}

\begin{assumption}[Unbiased]
\label{ass:soda-stochastic}
Let $\mathcal F_k$ be the natural filtration. 
The gradients satisfy
\[
\mathbb E[g^k \mid \mathcal F_{k-1}] \in \partial f(y^k).
\]
\end{assumption}
\begin{assumption}[Gradient variation]
\label{ass:noise}
The gradients satisfy, for $k\ge0$ and $\rho>0$,
\[
\mathbb E\!\left[\norm{g^k-g^{k-1}}_*^2\right]
\le
\rho\,\mathbb E\!\left[\norm{\nabla f(y^k)-\nabla f(y^{k-1})}_*^2\right] + \sigma^2.
\]
\end{assumption}
\begin{remark}
\Cref{ass:noise} is satisfied under unbiasedness (\Cref{ass:soda-stochastic}) and bounded variance,
$\mathbb E[\norm{g^k-\nabla f(y^k)}_*^2]\le v^2$, with $\rho=3$ and
$\sigma^2=6v^2$. 
The noise condition in \citet[Cor. 2]{defazio2024road} is recovered in the Euclidean case with $\rho=1$.
\end{remark}

Our proof primarily builds on \citet{defazio2024road}; however, rather than combining primal extrapolation with an adaptive version of Optimistic Mirror Descent, we use Optimistic Dual Averaging \citet{rakhlin2013online} as the underlying no-regret algorithm.
\looseness=-1

\begin{toappendix}
We first state the regret guarantee for ODA, and then employ it to analyze SODA.
The ODA analysis is based on a non-adaptive version of \citet{rakhlin2013online}, which is classical (see e.g., \citep[Sec.~7.12]{Orabona2019}).
The smooth
online-to-batch Lemma below builds on the argument of
\citet{defazio2024road} and extends it to non-Euclidean norms and slightly relaxes the noise condition.

\begin{lemma}[ODA regret]
\label{lem:oda-regret}
Consider the online-learning indexing of \ref{eq:ODA}, with the gradient term
$\nabla f(z^{k},\xi_{k})$ replaced by $g^{k}$. Choose positive weights
$a_0,\dots,a_{n-1}$ and write
\[
A_k := \textstyle\sum_{i=0}^{k} a_i,
\qquad A_{-1}:=0.
\]
We use the initial-hint convention $g^{-1}=0$ and set
$z^0\in\partial h^*(0)$. For $k=0,\dots,n-2$, choose
\[
\alpha_k = \tfrac{a_k}{A_k},
\quad
\bar\alpha_k = \tfrac{a_{k+1}}{A_{k+1}},
\quad
\gamma_k=\eta \tfrac{A_k}{1-\bar\alpha_k}
\]
for some $\eta>0$.
Let $h$ be a proper closed $\mu$-strongly convex regularizer on $\mathcal X$.
Then, for every
$x \in \mathcal X$,
\begin{equation}
\label{eq:oda-regret}
\textstyle\sum_{k=0}^{n-1} a_k \braket{g^k, z^k-x}
\le
\frac{h(x)-\inf_{u \in \mathcal X}h(u)}{\eta}
+
\frac{\eta}{2\mu}\textstyle\sum_{k=0}^{n-1} a_k^2 \norm{g^k-g^{k-1}}_*^2.
\end{equation}
\end{lemma}

\begin{proof}
By the choice of $\gamma_k$ and $\bar\alpha_k$ we have, for
$k=0,\dots,n-2$,
\[
\gamma_k(1-\bar\alpha_k)=\eta A_k,
\quad\text{and}\quad
\gamma_k\bar\alpha_k=\eta a_{k+1}
\]
Combined with the choice of $\alpha_k$, the \ref{eq:ODA} update can then be written as
\[
z^k\in
\partial h^*\!\left(
-\eta\textstyle\sum_{i=0}^{k-1}a_i g^i
-\eta a_k g^{k-1}
\right),
\qquad k=0,\dots,n-1,
\]
where the sum is empty when $k=0$ and the case $k=0$ is exactly
$z^0\in\partial h^*(0)$ because $g^{-1}=0$.
Set
\[
\theta_k := -\eta \textstyle\sum_{i=0}^{k} a_i g^i,
\qquad
\theta_{-1}:=0,
\qquad
\hat\theta_k := \theta_{k-1} - \eta a_k g^{k-1}.
\]
Thus $z^k \in \partial h^*(\hat\theta_k)$ and
\[
\theta_k = \hat\theta_k - \eta a_k(g^k-g^{k-1}).
\]
Since $h$ is $\mu$-strongly convex, $h^*$ is $1/\mu$-smooth, hence
\[
h^*(\theta_k)
\le
h^*(\hat\theta_k)
-
\eta a_k \braket{g^k-g^{k-1}, z^k}
+
\tfrac{\eta^2 a_k^2}{2\mu}\norm{g^k-g^{k-1}}_*^2.
\]
By Fenchel--Young,
\[
\braket{\hat\theta_k, z^k} = h(z^k)+h^*(\hat\theta_k),
\qquad
\braket{\theta_k, x} \le h(x)+h^*(\theta_k).
\]
Therefore
\[
\begin{aligned}
\eta a_k \braket{g^k, z^k-x}
&=
\eta a_k \braket{g^k-g^{k-1}, z^k}
+
\eta a_k \braket{g^{k-1}, z^k}
-
\eta a_k \braket{g^k, x}
\\
&=
\eta a_k \braket{g^k-g^{k-1}, z^k}
+
\braket{\theta_{k-1}-\hat\theta_k, z^k}
+
\braket{\theta_k-\theta_{k-1}, x}
\\
&=
\eta a_k \braket{g^k-g^{k-1}, z^k}
+
\braket{\theta_{k-1}, z^k}
-
\braket{\hat\theta_k, z^k}
+
\braket{\theta_k-\theta_{k-1}, x}
\\
&=
\eta a_k \braket{g^k-g^{k-1}, z^k}
+
\braket{\theta_{k-1}, z^k}
-
h(z^k)
-
h^*(\hat\theta_k)
+
\braket{\theta_k-\theta_{k-1}, x}
\\
&\le
h^*(\theta_{k-1}) - h^*(\theta_k)
+
\tfrac{\eta^2 a_k^2}{2\mu}\norm{g^k-g^{k-1}}_*^2
+
\braket{\theta_k-\theta_{k-1}, x}.
\end{aligned}
\]
Summing over $k=0,\dots,n-1$ yields
\[
\eta \textstyle\sum_{k=0}^{n-1} a_k \braket{g^k, z^k-x}
\le
h^*(0)-h^*(\theta_{n-1})
+
\braket{\theta_{n-1}, x}
+
\frac{\eta^2}{2\mu}\textstyle\sum_{k=0}^{n-1} a_k^2 \norm{g^k-g^{k-1}}_*^2.
\]
Finally, $h^*(0)=-\inf_{u\in\mathcal X}h(u)$ and
$\braket{\theta_{n-1}, x} - h^*(\theta_{n-1}) \le h(x)$, so dividing by $\eta$
proves the claim.
\end{proof}

\subsection{Bounded gradients}

Under bounded gradient we rely on the following online-to-batch conversion which holds for any primal extrapolation parameter $\bar\lambda_k \in [0,1]$.
\begin{lemma}[SODA online-to-batch]
\label{lem:soda-otb}
Consider \ref{eq:SODA} with positive weights $a_0,\dots,a_{n-1}$ and
\[
\lambda_{k-1} = \tfrac{a_k}{A_k},
\qquad
A_k := \textstyle\sum_{i=0}^{k} a_i,
\qquad
k=1,\dots,n-1.
\]
Suppose \Cref{ass:sample-convex,ass:soda-stochastic} hold.
Then, for every $x \in \mathcal X$,
\begin{equation}
\label{eq:soda-otb}
A_{n-1} \, \mathbb E[f(x^{n-1})-f(x)]
\le
\mathbb E\left[\textstyle\sum_{k=0}^{n-1} a_k \braket{g^k, z^k-x}\right].
\end{equation}
\end{lemma}

\begin{proof}
The result is obtained as an
application of \citet[Thm.~2]{defazio2024road} with
\[
w_t=a_k,
\quad
z_{t+1}=z^{k},
\quad
\beta_{t+1} = \bar\lambda_{k}
\quad
\]
for which 
\[
x_{t+1}=x^k,
\qquad
y_{t+1}=(1-\bar\lambda_{k-1})x^k+\bar\lambda_{k-1} z^k=y^k.
\]
\end{proof}

\end{toappendix}
\begin{toappendix}
\begin{assumption}[Bounded gradients]
\label{ass:bounded-gradients}
The gradients satisfies $\norm{g^k}_* \le G$ for all $k\ge 0$.
\end{assumption}
\begin{correp}[Convergence under bounded gradients]
\label{cor:soda-nonacc}
Let $x^\star\in\argmin_{x\in\mathcal X} f(x)$ and let
$R_\star:=h(x^\star)-\inf h$. Consider
\ref{eq:SODA} with a fixed regularizer $h_k \equiv h$. For every
$k=0,\dots,n-1$, choose
\begin{align*}
\alpha_k = \tfrac{1}{k+1},
\qquad
\lambda_k = \tfrac{1}{k+2},
\qquad
\bar\alpha_k = \tfrac{1}{k+2},
\qquad
\bar\lambda_k \in [0,1],
\qquad
\gamma_k = \eta(k+2),
\qquad
\eta = \tfrac{\sqrt{\mu R_\star}}{\sqrt{2}\,G\,\sqrt n}.
\end{align*}
Let $h$ be $\mu$-strongly convex with respect to $\norm{\cdot}$. Suppose
\Cref{ass:sample-convex,ass:soda-stochastic,ass:bounded-gradients} hold.
Then, for every $n \ge 1$,
\[
\mathbb E[f(x^{n-1})-f(x^\star)]
\le
\frac{2\sqrt{2}\,G\,\sqrt{R_\star}}{\sqrt{\mu n}}.
\]
\end{correp}
\begin{proof}
Choose the weights $a_k \equiv 1$, so that $A_k = k+1$ for
$k=0,\dots,n-1$. Then
\[
\alpha_k=\tfrac{1}{k+1}=\tfrac{a_k}{A_k},
\qquad
\bar\alpha_k=\tfrac{1}{k+2}=\tfrac{a_{k+1}}{A_k+a_{k+1}},
\qquad
\gamma_k(1-\bar\alpha_k)
=
\eta A_k,
\]
\[
\gamma_k\bar\alpha_k
=
\eta(k+2)\tfrac{1}{k+2}
=
\eta
=
\eta a_{k+1},
\]
for $k=0,\dots,n-2$, while
$\lambda_{k-1}=1/(k+1)=a_k/A_k$ for $k=1,\dots,n-1$. Therefore
\Cref{lem:oda-regret} and \Cref{lem:soda-otb} yield
\[
\mathbb E[f(x^{n-1})-f(x^\star)]
\le
\tfrac{R_\star}{\eta n}
+ 
\tfrac{\eta}{2\mu n}\textstyle\sum_{k=0}^{n-1}
\mathbb E\!\left[\norm{g^k-g^{k-1}}_*^2\right].
\]
Since $g^{-1}=0$ by convention and $\norm{g^k}_* \le G$, the sum is at most $4nG^2$.
Hence
\[
\mathbb E[f(x^{n-1})-f(x^\star)]
\le
\tfrac{R_\star}{\eta n}
+
\tfrac{2\eta G^2}{\mu}.
\]
Substituting the stated choice of $\eta$ gives the claim.
\end{proof}

\subsection{Gradient Lipschitz}

The following refinement of \Cref{lem:soda-otb} is used to exploit smoothness of $f$.
It follows the argument in \citet[Thm. 5]{defazio2024road} directly, but allows for an non-Euclidean norm.

For differentiable $f$, we write the objective Bregman divergence as
\[
D_f(u,v) := f(u)-f(v)-\braket{\nabla f(v), u-v}.
\]

\begin{lemma}[SODA online-to-batch smooth refinement {\citep[Thm. 5]{defazio2024road}}]
\label{lem:soda-otb-smooth}
Suppose the assumptions of \Cref{lem:soda-otb} hold, and additionally that $f$
is $L$-smooth with respect to $\norm{\cdot}$ (\Cref{ass:smooth}). If
\[
\bar\lambda_{k-1} \le \tfrac{a_k}{10A_k}
\qquad
\text{for all } k=1,\dots,n-1,
\]
then, for every $x \in \mathcal X$, \ref{eq:SODA} satisfies
\begin{equation}
\label{eq:soda-otb-smooth}
\begin{aligned}
A_{n-1} \, \mathbb E[f(x^{n-1})-f(x)]
&\le
\mathbb E\left[\textstyle\sum_{k=0}^{n-1} a_k \braket{g^k, z^k-x}\right] \\
&\qquad -
\tfrac{1}{6L}\textstyle\sum_{k=0}^{n-1} A_{k-1}
\mathbb E\!\left[\norm{\nabla f(y^k)-\nabla f(y^{k-1})}_*^2\right],
\end{aligned}
\end{equation}
where $A_{-1}=0$.
\end{lemma}
\begin{proof}
The choice
$\lambda_{k-1}=a_k/A_k$ implies
\[
x^k = \tfrac{1}{A_k}\textstyle\sum_{i=0}^k a_i z^i,
\qquad
k=0,\dots,n-1.
\]
For $k=0,\dots,n-1$, following \citet[Thm. 5]{defazio2024road}, a direct expansion gives
\[
\begin{aligned}
A_k\bigl(f(x^k)-f(x)\bigr)
- A_{k-1}\bigl(f(x^{k-1})-f(x)\bigr)
	&= 
	a_k\braket{\nabla f(y^k), z^k-x}
	\\
	&\qquad
	-
	\tfrac{a_k}{\bar\lambda_{k-1}} D_f(y^k,x^k)
	-
	\tfrac{a_k(1-\bar\lambda_{k-1})}{\bar\lambda_{k-1}} D_f(x^k,y^k)
    \\
	&\qquad
	-
	A_{k-1} D_f(x^{k-1},x^k)
    -
    a_k D_f(x,y^k),
\end{aligned}
\]
with the convention that the terms involving $A_{-1}$ vanish, and that the two
terms involving $\bar\lambda_{k-1}^{-1}$ vanish when $k=0$ or
$\bar\lambda_{k-1}=0$, since then $y^k=x^k$.
Summing over $k=0,\dots,n-1$, taking expectations, and using
$\mathbb E[g^k \mid \mathcal F_{k-1}] = \nabla f(y^k)$ gives
\[
\begin{aligned}
A_{n-1} \mathbb E[f(x^{n-1})-f(x)]
&\le
\mathbb E\left[\textstyle\sum_{k=0}^{n-1} a_k \braket{g^k, z^k-x}\right]
\\
&\qquad
- \textstyle\sum_{k=0}^{n-1}
\Bigl(
\tfrac{a_k}{\bar\lambda_{k-1}} \mathbb E D_f(y^k,x^k)
+
\tfrac{a_k(1-\bar\lambda_{k-1})}{\bar\lambda_{k-1}} \mathbb E D_f(x^k,y^k)
\Bigr)
\\
&\qquad
- \textstyle\sum_{k=0}^{n-1}
\Bigl(
A_{k-1}\mathbb E D_f(x^{k-1},x^k)
+ a_k\mathbb E D_f(x,y^k)
\Bigr).
\end{aligned}
\numberthis \label{eq:expectation}
\]
For convex $L$-smooth $f$, co-coercivity gives
$D_f(u,v)\ge \frac{1}{2L}\norm{\nabla f(u)-\nabla f(v)}_*^2$, so
\begin{align}
\tfrac{a_k}{\bar\lambda_{k-1}} D_f(y^k,x^k)
+
\tfrac{a_k(1-\bar\lambda_{k-1})}{\bar\lambda_{k-1}} D_f(x^k,y^k)
&\geq
\tfrac{a_k(2-\bar\lambda_{k-1})}{2L\bar\lambda_{k-1}}
\|\nabla f(y^k) - \nabla f(x^k)\|_*^2  \label{eq:breg1}
\\
A_{k-1} D_f(x^{k-1},x^k)
+
a_k D_f(x,y^k)
&\geq \tfrac{A_{k-1}}{2L}\|\nabla f(x^{k-1}) - \nabla f(x^k)\|_*^2  \label{eq:breg2}
\end{align}
Developing \eqref{eq:breg2}
with Young's inequality,
$-\|a+b+c\|_*^2 \leq -(1- 2/\tau)\|a\|_*^2 + (2\tau - 1)(\|b\|_*^2 + \|c\|_*^2)$ for $\tau=3$, we have
\begin{equation}\label{eq:young}
\begin{aligned}
\tfrac{A_{k-1}}{2L}\|\nabla f(x^{k-1}) - \nabla f(x^k)\|_*^2
&\geq 
\tfrac{A_{k-1}}{6L} \|\nabla f(y^{k})-\nabla f(y^{k-1})\|_*^2 \\
& \quad -\tfrac{5A_{k-1}}{2L}(\|\nabla f(y^{k})-\nabla f(x^{k})\|_*^2 + \|\nabla f(y^{k-1})-\nabla f(x^{k-1})\|_*^2)
\end{aligned}
\end{equation}
Summing \eqref{eq:breg1}, \eqref{eq:breg2} and using \eqref{eq:young} and that $\bar\lambda_{k-1} \le a_k/(10A_k)$ we have,
\[
\begin{aligned}
&-\tfrac{a_k}{\bar\lambda_{k-1}} D_f(y^k,x^k)
-
\tfrac{a_k(1-\bar\lambda_{k-1})}{\bar\lambda_{k-1}} D_f(x^k,y^k)
-
A_{k-1} D_f(x^{k-1},x^k)
-
a_k D_f(x,y^k)
\\
&\qquad\le
-
\tfrac{A_{k-1}}{6L}
\norm{\nabla f(y^k)-\nabla f(y^{k-1})}_*^2
-
\tfrac{5A_k}{2L}\norm{\nabla f(y^k)-\nabla f(x^k)}_*^2
\\
&\qquad\qquad
+
\tfrac{5A_{k-1}}{2L}\norm{\nabla f(y^{k-1})-\nabla f(x^{k-1})}_*^2.
\end{aligned}
\]
Summing over $k=0,\dots,n-1$ makes the final two terms telescope. Dropping the
nonpositive telescoping remainder and combining with \eqref{eq:expectation} gives \eqref{eq:soda-otb-smooth}.
\end{proof}
\end{toappendix}

\begin{correp}[Convergence under $L$-smoothness]
\label{cor:soda-nonacc-general}
Let $x^\star\in\argmin_{x\in\mathcal X} f(x)$ and let
$R_\star:=h(x^\star)-\inf h$. Consider
\ref{eq:SODA} with a fixed regularizer $h_k \equiv h$.
For every
$k=0,\dots,n-1$, choose
\begin{align*}
\alpha_k = \tfrac{1}{k+1},
\qquad
\bar\alpha_k = \lambda_k = \tfrac{1}{k+2},
\qquad
\bar\lambda_k \le \tfrac{\lambda_k}{10},
\qquad
\gamma_k = \eta(k+2),
\qquad
\eta = \min\left\{
\tfrac{\mu}{6\rho L},
\tfrac{\sqrt{\mu R_\star}}{\sigma\sqrt n}
\right\}.
\end{align*}
Suppose \Cref{ass:sample-convex,ass:soda-stochastic,ass:smooth,ass:noise}
hold and that $h$ is $\mu$-strongly convex with respect to $\norm{\cdot}$.
Then, for every $n\ge1$,
\[
\mathbb E[f(x^{n-1})-f(x^\star)]
=
O\!\left(
\tfrac{(\rho+1)LR_\star}{\mu n}
+
\tfrac{\sigma\sqrt{R_\star}}{\sqrt{\mu n}}
\right).
\]
\end{correp}
\begin{proof}
Choose the weights $a_k \equiv 1$, so that $A_k=k+1$ for
$k=0,\dots,n-1$. Then
\[
\alpha_k=\tfrac{1}{k+1}=\tfrac{a_k}{A_k},
\qquad
\bar\alpha_k=\tfrac{1}{k+2}=\tfrac{a_{k+1}}{A_{k+1}},
\qquad
\gamma_k(1-\bar\alpha_k)
=
\eta A_k,
\]
\[
\gamma_k\bar\alpha_k
=
\eta(k+2)\tfrac{1}{k+2}
=
\eta
=
\eta a_{k+1},
\]
for $k=0,\dots,n-2$, while
$\lambda_{k-1}=1/(k+1)=a_k/A_k$ for $k=1,\dots,n-1$.
Under $L$-smoothness, \Cref{lem:soda-otb-smooth} gives
\[
\begin{aligned}
n\,\mathbb E[f(x^{n-1})-f(x^\star)]
&\le
\tfrac{R_\star}{\eta}
+
\tfrac{\eta}{2\mu}\textstyle\sum_{k=0}^{n-1}
\mathbb E\!\left[\norm{g^k-g^{k-1}}_*^2\right]
\\
&\qquad
-
\tfrac{1}{6L}\textstyle\sum_{k=0}^{n-1} k\,
\mathbb E\!\left[\norm{\nabla f(y^k)-\nabla f(y^{k-1})}_*^2\right].
\end{aligned}
\]
Using \Cref{ass:noise} with the boundary conventions $g^{-1}=0$ and
$\nabla f(y^{-1})=\nabla f(x^\star)$, the startup term satisfies
\[
\mathbb E\!\left[\norm{g^0-g^{-1}}_*^2\right]
\le
\rho\,\mathbb E\!\left[\norm{\nabla f(y^0)-\nabla f(x^\star)}_*^2\right]
+
\sigma^2.
\]
For $k=1,\dots,n-1$, the same assumption gives
\[
\mathbb E\!\left[\norm{g^k-g^{k-1}}_*^2\right]
\le
\rho\,\mathbb E\!\left[\norm{\nabla f(y^k)-\nabla f(y^{k-1})}_*^2\right]
+
\sigma^2.
\]
Substituting this estimate into the previous bound gives
\[
\begin{aligned}
n\,\mathbb E[f(x^{n-1})-f(x^\star)]
&\le
\tfrac{R_\star}{\eta}
+
\tfrac{\eta n\sigma^2}{2\mu}
+
\tfrac{\eta\rho}{2\mu}
\mathbb E\!\left[\norm{\nabla f(y^0)-\nabla f(x^\star)}_*^2\right]
\\
&\qquad
+
\textstyle\sum_{k=1}^{n-1}
\left(\tfrac{\eta \rho}{2\mu} - \tfrac{k}{6L}\right)
\mathbb E\!\left[\norm{\nabla f(y^k)-\nabla f(y^{k-1})}_*^2\right].
\end{aligned}
\]
If $\eta \le \mu/(6\rho L)$, then the coefficients in the final sum are
nonpositive and the final sum can be dropped. Thus
we obtain
\[
\mathbb E[f(x^{n-1})-f(x^\star)]
\le
\tfrac{R_\star}{\eta n}
+
\tfrac{\eta\sigma^2}{2\mu}
+
\tfrac{\eta\rho}{2\mu n}
\mathbb E\!\left[\norm{\nabla f(y^0)-\nabla f(x^\star)}_*^2\right].
\]
The displayed initialization bound gives
$\mathbb E[\norm{\nabla f(y^0)-\nabla f(x^\star)}_*^2]\le 2L^2R_\star/\mu$. Since
$\eta\le\mu/(6\rho L)$, the startup term is
$O(LR_\star/(\mu n))$. The remaining two terms are optimized by the stated
choice of $\eta$, yielding the claim.
\end{proof}
\paragraph{Consequences for practice}

The SODA wrapper in \Cref{alg:SODA} directly uses the theoretically suggested choices of
$\lambda_k$ and $\gamma_k$. 
Taking the primal extrapolation constant $\bar\lambda_k$ small is used to exploit smoothness to cancel the gradient-variation term from \Cref{ass:noise}.
Since $\bar\lambda_k$ has no effect on the rate once it is small enough, we can conveniently set $\bar{\lambda}_k=0$ in practice, matching the
``modernized'' parameterization in \Cref{sec:method} and recovering Muon, Lion,
and NAdam. Under the bounded-gradient analysis of \Cref{cor:soda-nonacc}, the
larger range $\bar\lambda_k\in[0,1]$ is also admissible.

The smooth analysis also requires the optimistic averaging parameter
$\bar\alpha_k$ to be slightly smaller than $\alpha_k$. This is consistent with
common choices: Lion often uses $\alpha_k=0.1$ with
$\bar\alpha_k\in\{0.05,0.01\}$, while Muon takes
$\alpha_k=\bar\alpha_k$. Finally, the rate depends on the initial regularizer
gap $R_\star$. For $h(x)=\|x-z^0\|^2$, we have
$R_\star=\|x^\star-z^0\|^2$, so choosing $z^0=0$ can be a poor anchor when the
solution is far from the origin. This motivates the centered choice in
\Cref{sec:method}.
\paragraph{Acceleration}
\Cref{cor:soda-acc} shows that SODA also admits an accelerated parameterization following \citet[Cor.~1]{defazio2024road}.
This is obtained by choosing
$\alpha_k = \nicefrac{2}{k+2}$,
$\bar\alpha_k=\lambda_k=\nicefrac{2}{k+3}$, and
$\bar\lambda_k \le \nicefrac{\lambda_k}{10}$.
Equivalently, this accelerated regime corresponds to using increasing weights
$a_k = k+1$ for $k\ge0$ and taking
$\alpha_k = \nicefrac{a_k}{\sum_{i=0}^k a_i}$ and
$\lambda_k = \nicefrac{a_{k+1}}{\sum_{i=0}^{k+1}a_i}$.
We further comment in \Cref{app:acceleration}.

\paragraph{Limitations}
Our analysis is convex. Although convex theory often remains empirically
informative for deep learning, as observed by
\citet{defazio2024road,schaipp2025surprising} and in our experiments, extending
the guarantees to less restrictive assumptions is an important direction.
A second limitation is the strong convexity requirement on the regularizer. For
a regularizer $h$ that is not strongly convex, one can instead use
$h_\tau(x):=h(x)+\tau\psi(x)$, where $\psi$ is $1$-strongly convex with respect
to the chosen norm and $\tau>0$. Substituting $h_\tau$ for $h$ and $\tau$ for
$\mu$ preserves the same $O(n^{-1/2})$ dependence on $n$, but at the cost of a
$1/\sqrt{\tau}$ factor in the constant.

\begin{toappendix}
\subsection{Accelerated rate}\label{app:acceleration}
\begin{correp}[Accelerated parameterization]
\label{cor:soda-acc}
Let $x^\star\in\argmin_{x\in\mathcal X} f(x)$ and let
$R_\star:=h(x^\star)-\inf h$. Consider
\ref{eq:SODA} with a fixed regularizer $h_k \equiv h$. For every
$k=0,\dots,n-1$, choose
\begin{align*}
\alpha_k = \tfrac{2}{k+2},
\qquad
\bar\alpha_k = \lambda_k = \tfrac{2}{k+3},
\qquad 
\bar\lambda_k \le \tfrac{\lambda_k}{10},
\\
\gamma_k = \eta \tfrac{(k+2)(k+3)}{2},
\qquad
\eta = \min\left\{\tfrac{\mu}{12\rho L}, \tfrac{\sqrt{\mu R_\star}}{\sigma n^{3/2}}\right\}.
\end{align*}
Assume \Cref{ass:sample-convex,ass:soda-stochastic,ass:smooth,ass:noise},
and that $h$ is $\mu$-strongly convex with respect to $\norm{\cdot}$.
Then, for every $n\ge 1$,
\[
\mathbb E[f(x^{n-1})-f(x^\star)]
=
O\left(
\frac{(\rho+1)LR_\star}{\mu n^2}
+
\frac{\sigma\sqrt{R_\star}}{\sqrt{\mu n}}
\right).
\]
\end{correp}
\begin{remark}
\Cref{cor:soda-acc} suggests that optimism ($\bar\alpha_k$) and primal averaging ($\lambda_k$) should be scheduled in tandem to obtain acceleration.
In the context of deep learning, this suggests that optimism, which is commonly referred to as Nesterov momentum, does not lead to acceleration on its own, but should be scheduled in combination with the weight decay ($\lambda_k$).
We leave this as interesting future work.
\end{remark}
\begin{proof}
Choose the weights $a_k = k+1$ for $k=0,\dots,n-1$. Then
\[
A_k=\textstyle\sum_{i=0}^{k} a_i=\tfrac{(k+1)(k+2)}{2},
\qquad
\frac{a_k}{A_k}=\frac{2}{k+2},
\qquad
\frac{a_{k+1}}{A_{k+1}}
=\frac{2}{k+3}.
\]
For $k=0,\dots,n-2$, the displayed parameter choices therefore satisfy
\[
\begin{split}
    \alpha_k=\tfrac{a_k}{A_k},
    \quad
    \lambda_k=\tfrac{a_{k+1}}{A_{k+1}},
    \quad 
    \bar\lambda_k\le \tfrac{a_{k+1}}{10A_{k+1}}, 
    \bar\alpha_k=\tfrac{a_{k+1}}{A_{k+1}}, 
    \gamma_k=\eta \tfrac{A_k}{1-\bar\alpha_k},
\end{split}
\]
\Cref{lem:oda-regret} and \Cref{lem:soda-otb-smooth} give
\begin{align*}
A_{n-1} \mathbb E[f(x^{n-1})-f(x^\star)]
&\le
\tfrac{R_\star}{\eta}
+
\tfrac{\eta}{2\mu}\textstyle\sum_{k=0}^{n-1} (k+1)^2
\mathbb E\!\left[\norm{g^k-g^{k-1}}_*^2\right]
\\
&\quad -
\tfrac{1}{6L}\textstyle\sum_{k=0}^{n-1} A_{k-1}
\mathbb E\!\left[\norm{\nabla f(y^k)-\nabla f(y^{k-1})}_*^2\right].
\end{align*}
Using \Cref{ass:noise} with the boundary conventions $g^{-1}=0$ and
$\nabla f(y^{-1})=\nabla f(x^\star)$, the startup term satisfies
\[
\mathbb E\!\left[\norm{g^0-g^{-1}}_*^2\right]
\le
\rho\,\mathbb E\!\left[\norm{\nabla f(y^0)-\nabla f(x^\star)}_*^2\right]
+
\sigma^2.
\]
For $k=1,\dots,n-1$, the same assumption controls the gradient differences,
and we obtain
\[
\begin{aligned}
A_{n-1} \mathbb E[f(x^{n-1})-f(x^\star)]
&\le
\tfrac{R_\star}{\eta}
+
\tfrac{\eta\rho}{2\mu}
\mathbb E\!\left[\norm{\nabla f(y^0)-\nabla f(x^\star)}_*^2\right]
+
\tfrac{\eta\sigma^2}{2\mu}\textstyle\sum_{k=0}^{n-1} (k+1)^2
\\
&\qquad
+
\textstyle\sum_{k=1}^{n-1}
\left(
\tfrac{\eta \rho}{2\mu}(k+1)^2
-
\tfrac{A_{k-1}}{6L}
\right)
\mathbb E\!\left[\norm{\nabla f(y^k)-\nabla f(y^{k-1})}_*^2\right],
\end{aligned}
\]
where we used
$A_{k-1}=\tfrac{k(k+1)}{2}\ge \tfrac{(k+1)^2}{4}$ for $k\ge 1$.
By the stated choice of $\eta$, the coefficient of the final term is
nonpositive, so it can be dropped. Since
\[
A_{n-1} = \Theta(n^2),
\qquad
\textstyle\sum_{k=1}^{n-1} (k+1)^2 = O(n^3),
\]
it follows that
\[
\mathbb E[f(x^{n-1})-f(x^\star)]
=
O\left(
\tfrac{R_\star}{\eta n^2}
+
\tfrac{\eta\sigma^2 n}{\mu}
+
\tfrac{\eta\rho}{\mu n^2}
\mathbb E\!\left[\norm{\nabla f(y^0)-\nabla f(x^\star)}_*^2\right]
\right).
\]
Now set
\[
\eta_1 := \tfrac{\mu}{12\rho L},
\qquad
\eta_2 := \tfrac{\sqrt{\mu R_\star}}{\sigma n^{3/2}},
\qquad
\eta = \min\{\eta_1,\eta_2\}.
\]
If $\eta=\eta_1$, then
\[
\tfrac{R_\star}{\eta n^2}
=
O\!\left(\tfrac{\rho LR_\star}{\mu n^2}\right),
\qquad
\tfrac{\eta \sigma^2 n}{\mu}
=
O\!\left(\tfrac{\sigma^2 n}{\rho L}\right)
\le
\tfrac{\sigma\sqrt{R_\star}}{\sqrt{\mu n}},
\]
because $\eta_1\le\eta_2$ implies
$\sigma n^{3/2}=O(\rho L\sqrt{R_\star/\mu})$.
The startup term is
$\eta\rho\mathbb E[\norm{\nabla f(y^0)-\nabla f(x^\star)}_*^2]/(\mu n^2)
=O(LR_\star/(\mu n^2))$.
If instead $\eta=\eta_2$, then
\[
\tfrac{R_\star}{\eta n^2}
=
\tfrac{\eta \sigma^2 n}{\mu}
=
\tfrac{\sigma\sqrt{R_\star}}{\sqrt{\mu n}},
\]
and the startup term is again $O(LR_\star/(\mu n^2))$ since
$\eta\le\eta_1$.
Combining the two cases gives the claim.
\end{proof}

\end{toappendix}

\begin{toappendix}
\section{Experiments}\label{app:experiments}
\end{toappendix}
\section{Experiments}\label{sec:experiments}
\nosectionappendix

\begin{figure}[t]
    \centering
    \includegraphics[width=0.45\linewidth]{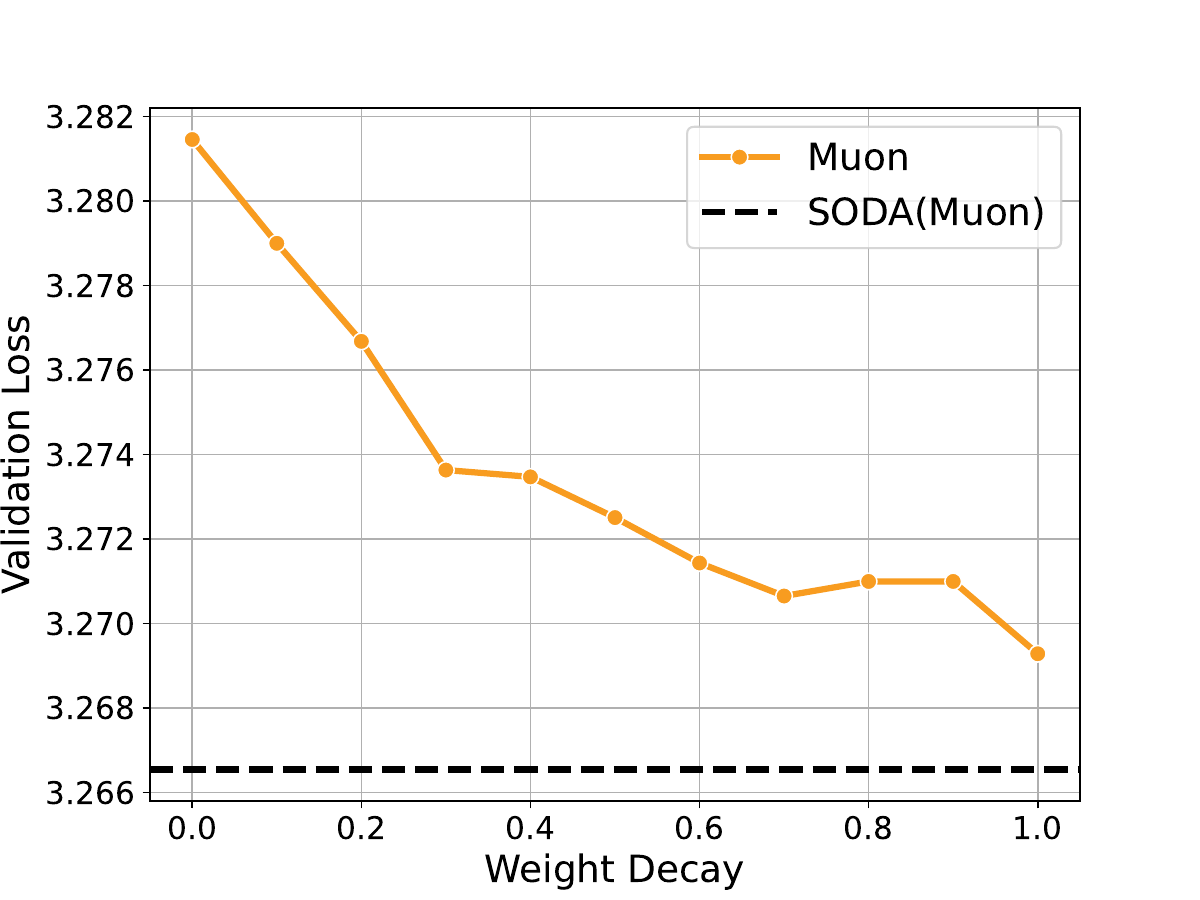}
    \hfill
    \includegraphics[width=0.45\linewidth]{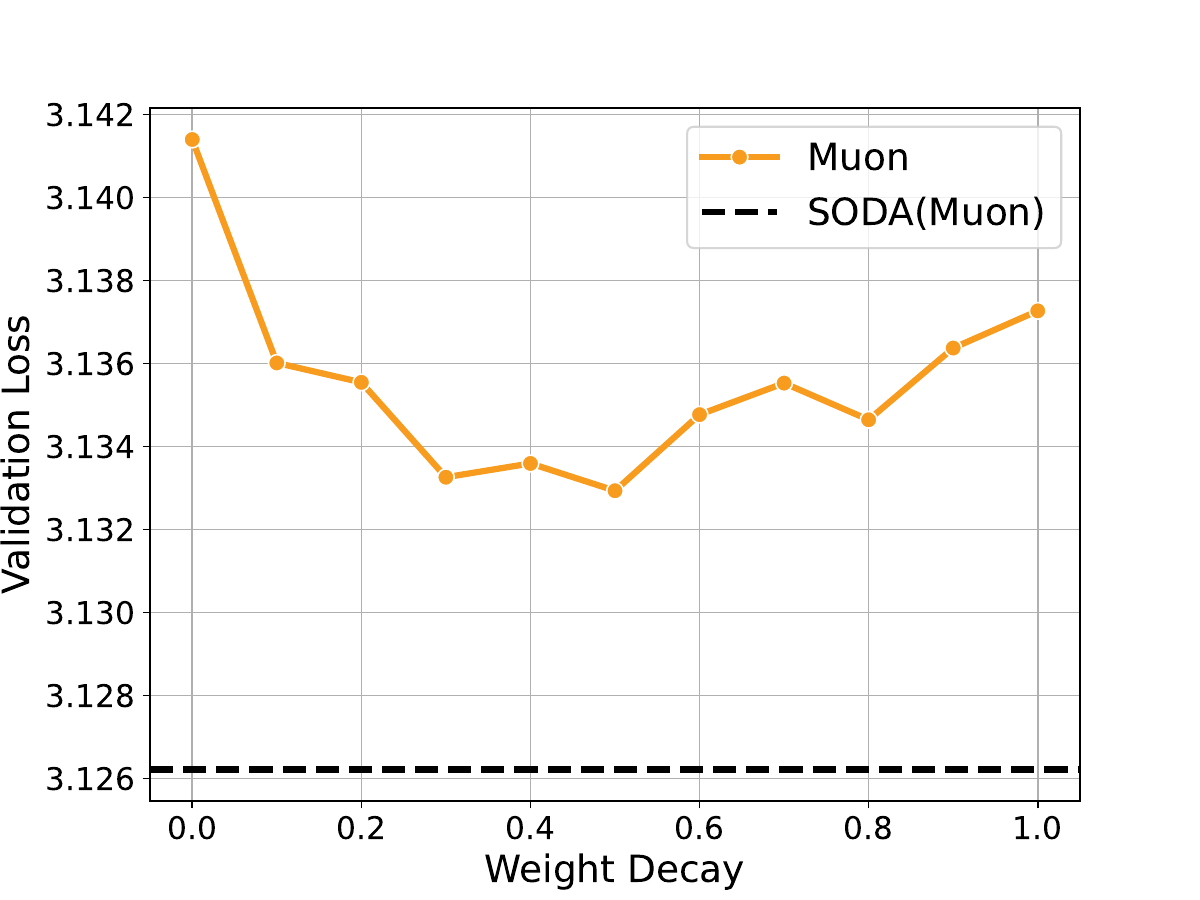}
    \caption{Muon with swept weight decay is outperformed by SODA(Muon), without any additional tuning, on 124M models trained for both $1 \times$ Chinchilla steps (left) and $4 \times$ Chinchilla steps (right).
    }
    \label{fig:wd_sweep}
\end{figure}

\begin{figure}[t]
    \centering
    \includegraphics[width=0.45\linewidth]{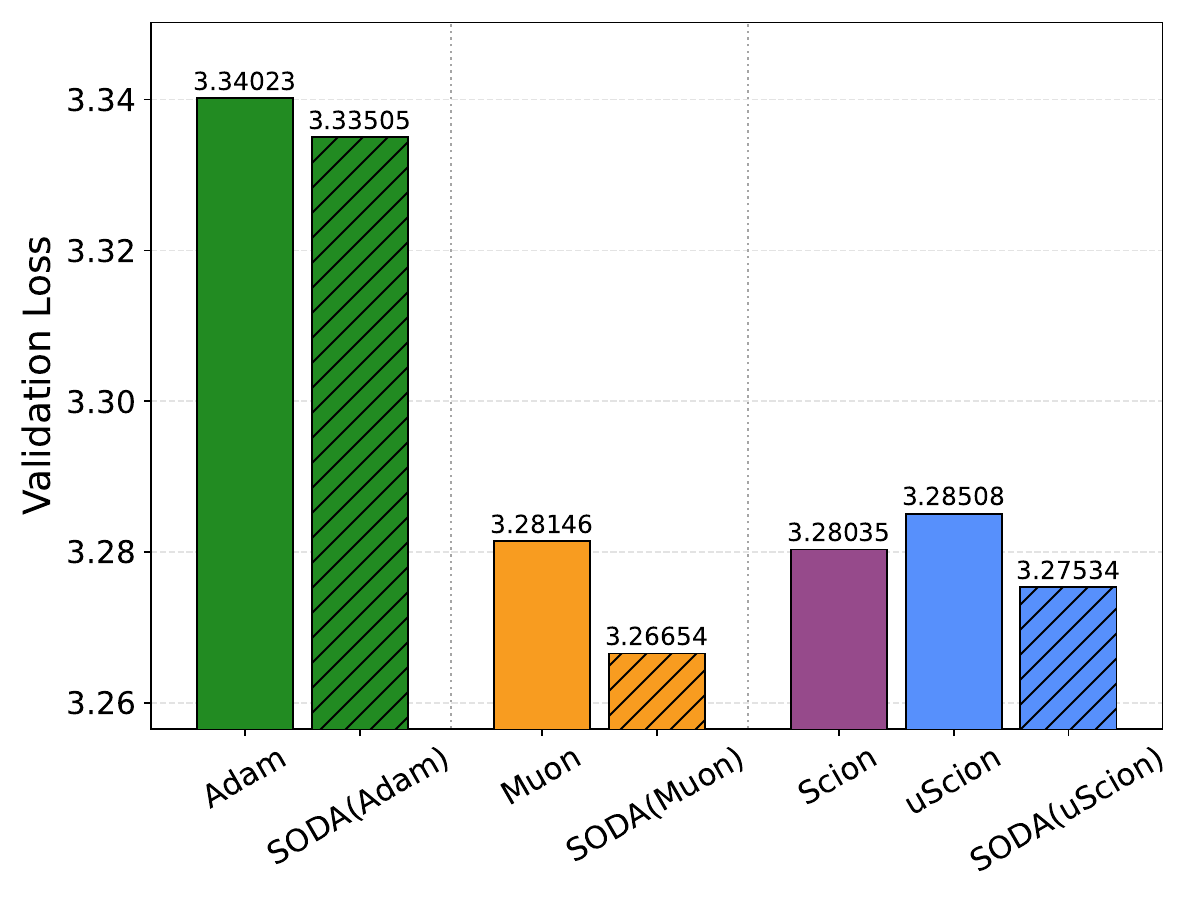}
    \hfill
    \includegraphics[width=0.45\linewidth]{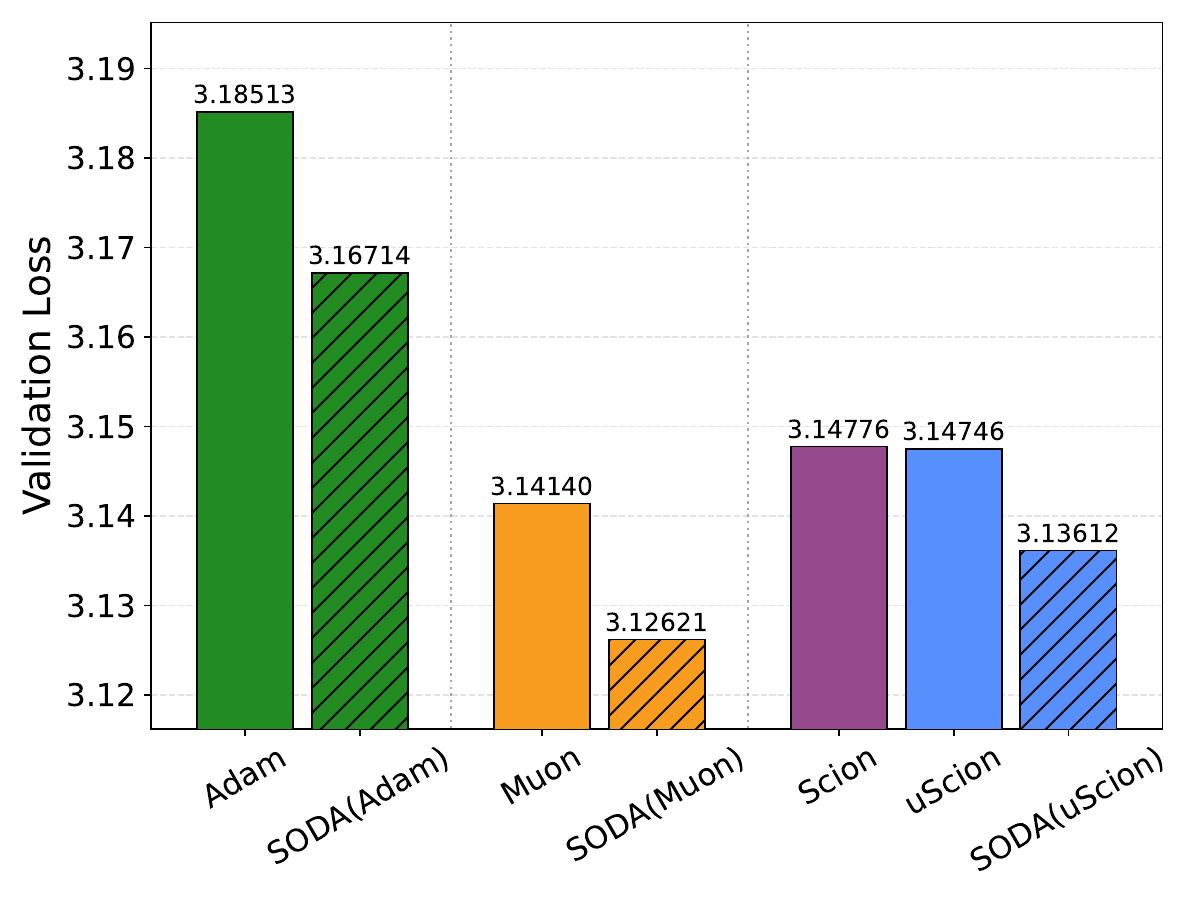}
    \caption{The SODA wrapper yields consistent improvement across various base optimizers without any additional tuning as illustrated on 124M model trained for both $1 \times$ Chinchilla steps (left) and $4 \times$ Chinchilla steps (right).
    }
    \label{fig:wrapper_optimizers}
\end{figure}

\begin{figure}[t]
    \centering
    \includegraphics[width=0.49\linewidth]{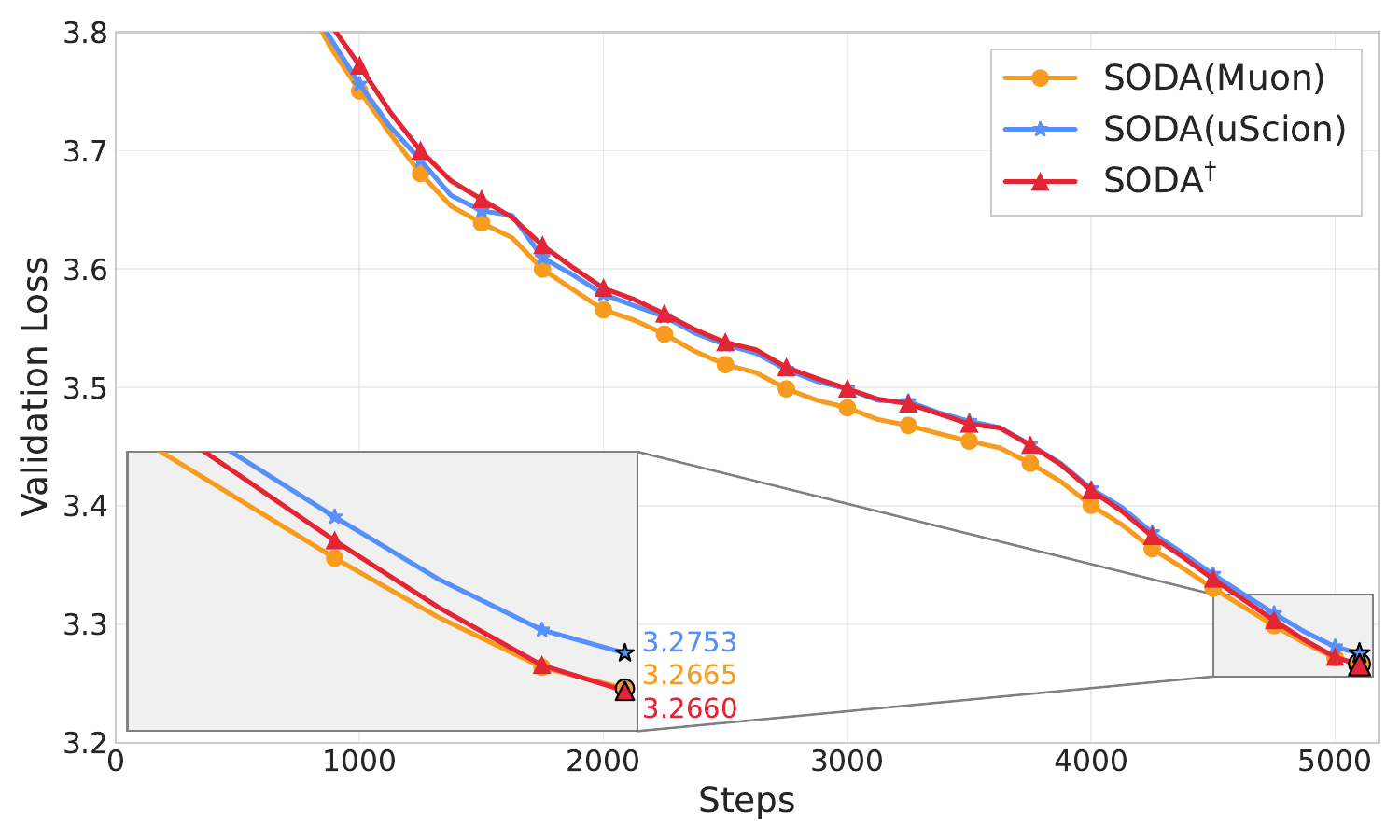}
    \hfill
    \includegraphics[width=0.49\linewidth]{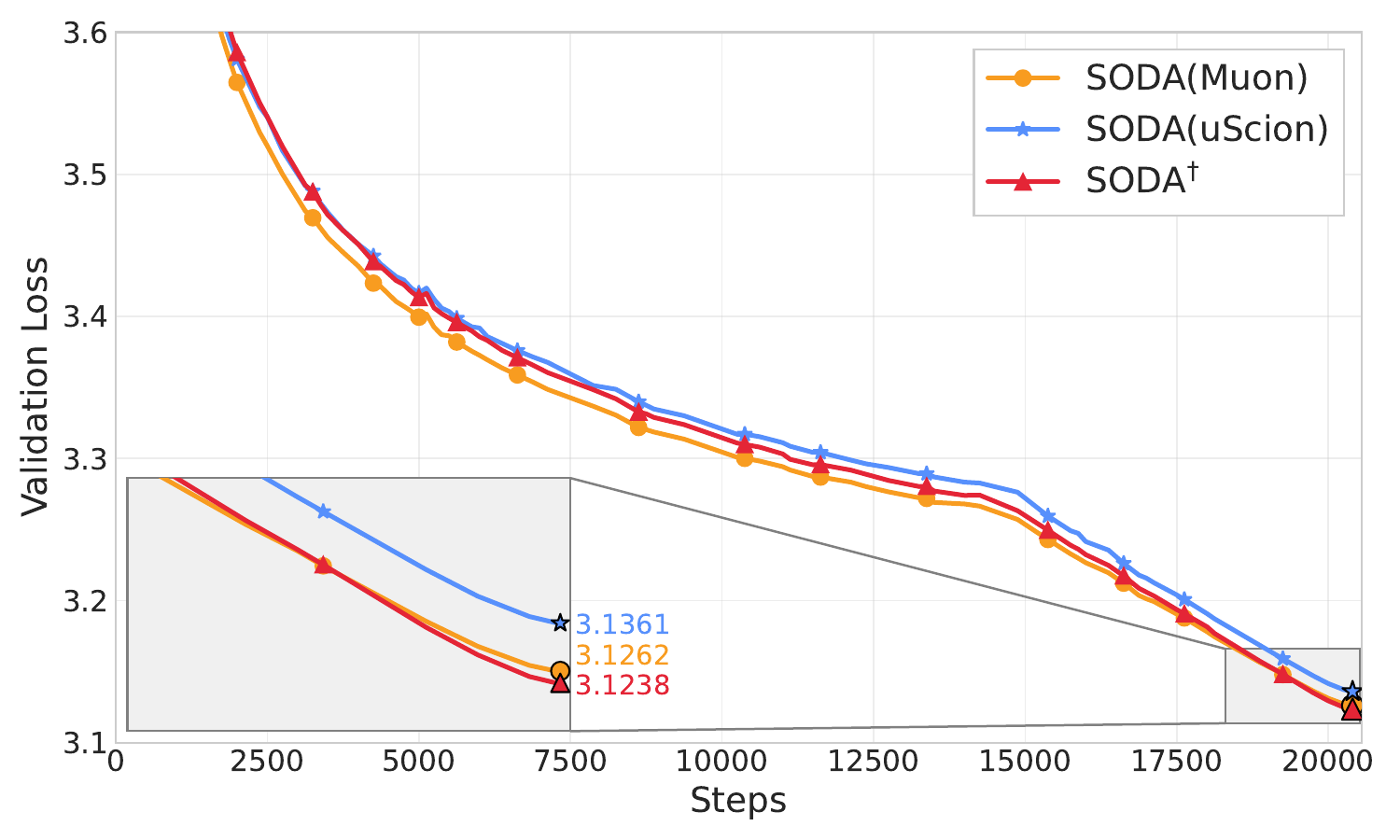}
    \caption{
    SODA with optimism (referred to as SODA$\dag$) is competitive with the best wrapped optimizer.
    In comparison with SODA(Muon), the configuration simplifies the method by replacing Adam with Lion and reusing the same hyperparameters for the momentum across all layers.
    }
    \label{fig:val_loss_trajectory}
\end{figure}

Throughout the experiments, Muon refers to the official implementation, which uses Adam for the first and last layer and Nesterov momentum for hidden layers (corresponding to optimism in \ref{eq:SODA} with $\bar\alpha_k=\alpha_k$).
The Scion optimizer, on the other hand, disables the Nesterov momentum ($\bar\alpha_k=0$) and uses Signum instead of Adam for the first and last layer.
Experiments are conducted on nanoGPT on FineWeb100 (see \Cref{tbl:hyperparams:nanoGPT} for full details).
In the experiments, the token budget is expressed in units of Chinchilla \citep{hoffmann2022training}, where $1\times$ Chinchilla corresponds to $20\times \#(\text{parameters})$.

\vspace{-0.5em}
\paragraph{Transfer across horizon}
Zero-shot hyperparameter transfer results, most notably $\mu$P, show that learning rates and related optimization hyperparameters can transfer reliably across width \citep{yang2022tensor}. 
However, these transfer results are typically demonstrated at a fixed training horizon. This leaves open the more practically important question of how weight decay should transfer across horizon, since for longer runs the optimal weight decay typically changes. %

To isolate the effect of the horizon, we keep the model width fixed and vary only the training horizon. 
We find that the optimal weight decay \emph{decreases} with the horizon for Muon (\textit{c.f.} \Cref{fig:wd_sweep}),
thus demonstrating that the $1/\texttt{model\_size}$ choice made in \citet{xiao2024rethinking,qiu2025hyperparameter}, which is constant in the horizon, would be suboptimal outside the Chinchilla scaling rule where the model size and horizon is scaled proportionally \citep{hoffmann2022training}.

Furthermore, we find that the theoretically motivated $1/(k+2)$ weight decay scaling used by the SODA wrapper (\Cref{alg:SODA}) consistently leads to further improvement without requiring any additional tuning. In \cref{fig:wd_sweep}, SODA outperforms the base optimizer (Muon) even when the latter is given its best-tuned weight decay, showing that the gain is not simply due to a favorable retuning of that hyperparameter. Taken together, these results suggest that SODA provides a principled mechanism for transferring weight decay across horizon without the need for tuning the weight decay even of the smaller proxy model.
The SODA wrapper relies on a non-zero center iterate $z^0$, which we further ablate the importance of in \Cref{fig:val_loss_z0}.

\paragraph{SODA Wrapper}
The benefit of the SODA Wrapper (\Cref{alg:SODA}) is not exclusive to Muon. 
We additionally apply the wrapper across Adam and Scion and observe consistent improvements across training horizons in \Cref{fig:wrapper_optimizers}, notably without introducing any additional tuning.

\begin{wrapfigure}{r}{0.48\textwidth}
    \vspace{-0.5em}
    \centering
    \includegraphics[width=\linewidth]{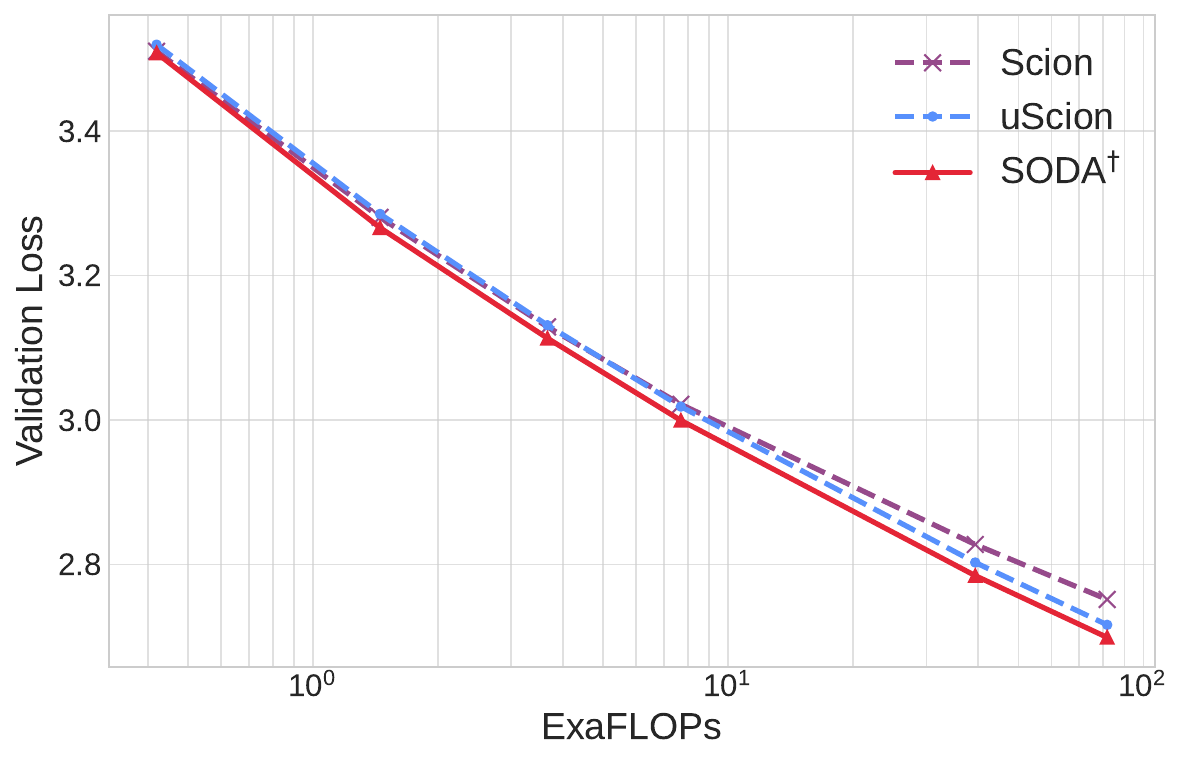}
    \caption{SODA is effective under 1$\times$ Chinchilla scaling and the benefit increases with scale.}
    \label{fig:scale_width_horizon}
\end{wrapfigure}

\vspace{-0.5em}
\paragraph{Optimism and SODA$^\dag$}
Considering the benefit of optimism (Muon) in the overtrained regime of \Cref{fig:wrapper_optimizers} we systematically investigate the impact of optimism in SODA, reported in \Cref{tab:optimism} of \Cref{app:experiments}.
We use the $\ell_\infty$-norm for the input and output layer and spectral norm for hidden layers following \citet{pethick2025trainingdeeplearningmodels}.
SODA with optimism, which borrows the same optimized configuration as the Muon baseline ($\alpha_k=\bar \alpha_k=0.05$), shows the best performance under both $1\times$ and $4\times$ Chinchilla.
This SODA configuration notably removes the use of Adam in Muon and uses the same momentum hyperparameter across all layers in contrast, greatly simplifying tuning.
We mark SODA with this specific optimistic setting as SODA$^\dag$. 
\Cref{fig:val_loss_trajectory} shows validation loss compared with the best two settings found in \Cref{fig:wrapper_optimizers}.

\vspace{-0.5em}
\paragraph{Transfer across horizon \& width}
In \cref{fig:scale_width_horizon}, we test SODA$^\dag$ on NanoGPT with between 64M and 1B parameters following Chinchilla-style scaling, where width and horizon grow proportionally. We choose (u)Scion as the base optimizer since its stepsize can be transferred along the width and horizon \citep{pethick2025trainingdeeplearningmodels}.
SODA$^\dag$ consistently outperforms Scion/uScion with different sizes of models and the gap becomes more apparent as the model grows larger.
A similar conclusion holds for SODA(uScion) (\textit{c.f.} \Cref{fig:scale_width_horizon_supp}). 

\begin{toappendix}
\subsection{Implementation}\label{app:implementation}

\paragraph{Norm choice}
Consider a neural network parameterized by weight matrices $W_1, \dots, W_L$, where $W_1$ is the first layer and $W_L$ is the last layer. 
Following \citet{pethick2025trainingdeeplearningmodels}, we use the max-norm across layers
\begin{equation*}
\|x\| = \max_l \tfrac{1}{\rho_l}\|W_l\|_{\alpha_l \rightarrow \beta_l}
\end{equation*}
where $\rho_l$ is the layerwise radii and $\|\cdot\|_{\alpha \rightarrow \beta}:= \sup_{x \neq 0} \|Ax\|_\beta/\|x\|_\alpha$ is an operator norm.
Under the max-norm the $\lmo$ decomposes layerwise.
Different layerwise operator norms can be found in \Cref{tbl:parameter:lmo:same-norm}.
If not otherwise noted, we use the operator norm choice $(\text{Sign} \rightarrow \text{Spectral} \rightarrow \text{Sign}$) for both Scion and SODA, which refers to using Sign for the first layer, Spectral for intermediary layers and Sign for the last layer.

\begin{table*}[t]
\centering
\caption{
  The linear minimization oracle ($\lmo$) for various norm choices.
  Define $\operatorname{col}_j(A):=[A]_{\cdot,j}$ and $\operatorname{row}_i(A):=[A]_{i,\cdot}$.
}
\label{tbl:parameter:lmo:same-norm}
\bgroup
\def\arraystretch{1.2}
\resizebox{\textwidth}{!}{
\begin{tabular}{|c|c|c|c|c|}
\hline
Weight matrix norm
    & $W_1$ (1-hot encoded) 
    & $W_1$ (image domain) 
    & $(W_\ell)_{\ell \in [2,...,L-1]}$
    & $W_L$
    \\
\hline
\hline
Spectral 
   ($\RMS \rightarrow \RMS$)
   & $-\sqrt{d_\mathrm{out}} UV^\top$
   & \multicolumn{3}{c|}{$-\sqrt{\nicefrac{d_\mathrm{out}}{d_\mathrm{in}}} UV^\top$}
 \\
 \hline
ColNorm 
  ($1 \rightarrow \RMS$)
  & $\operatorname{col}_j(W_1)\mapsto -\sqrt{d_\mathrm{out}}\tfrac{\operatorname{col}_j(W_1)}{\|\operatorname{col}_j(W_1)\|_2}$
  & \multicolumn{3}{c|}{$\operatorname{col}_j(W_\ell)\mapsto -\tfrac{\sqrt{d_\mathrm{out}}}{d_\mathrm{in}}\tfrac{\operatorname{col}_j(W_\ell)}{\|\operatorname{col}_j(W_\ell)\|_2}$}
  \\
\hline
RowNorm 
    ($\RMS \rightarrow \infty$)
   & $\operatorname{row}_i(W_1)\mapsto -\tfrac{\operatorname{row}_i(W_1)}{\|\operatorname{row}_i(W_1)\|_2}$
   & \multicolumn{3}{c|}{$\operatorname{row}_i(W_\ell)\mapsto -\tfrac{\operatorname{row}_i(W_\ell)}{\sqrt{d_\mathrm{in}}\|\operatorname{row}_i(W_\ell)\|_2}$}
 \\
 \hline
Sign
    ($1 \rightarrow \infty$)
  & $-\sign(W_1)$
  & \multicolumn{3}{c|}{$-\tfrac{1}{d_\mathrm{in}}\sign(W_\ell)$}
\\
\hline
\end{tabular}
}
\egroup
\end{table*}

\paragraph{SODA Wrapper}
For ease of use, the provided implementation of the SODA Wrapper (\Cref{alg:SODA}) can be used directly as a wrapper around any existing PyTorch optimizer. 
However, this ease of use comes at the cost of requiring additional memory since we need to recover the update delta $x^{k+1} - x^{k}$, which requires storing the previous iterate $x^{k}$.
For production, we therefore recommend one of three options:
\begin{enumerate}[label=(\roman*)]
\item modify the base optimizer to directly return the update delta $x^{k+1} - x^{k}$, 
\item directly implement the simplified wrapper expression \eqref{eq:SODA:wrapper:simplified} inside the base optimizer,
\item or use the \ref{eq:SODA} algorithm, which involves the dual momentum and the mirror map.
\end{enumerate}
For Adam, which uses stepsize warmup the SODA wrapper is only applied after the warmup phase.

\paragraph{Layer choice}
For all SODA Wrapper experiments, we wrap only the hidden layers, keeping the first/last layers without weight decay, following the choice of modded-nanogpt, which only sets weight decay for hidden layers. 
We further test this layer choice in \Cref{tab:wrapper_layers}.

For SODA(Muon), wrapping only hidden layers is consistently better than wrapping all layers under both $1\times$ and $4\times$ Chinchilla.
On the other hand, when using uScion with the norm choice ($\text{Sign} \rightarrow \text{Spectral} \rightarrow \text{Sign}$) as the base optimizer, wrapping all layers is competitive with wrapping only hidden layers.
It appears that the $1/(k+2)$ weight decay introduced by the SODA Wrapper improves performance in hidden layers, which benefit from weight decay, while not harming the first and last layers, which are typically more sensitive to it.
This suggests that also the non-wrapped version SODA with ($\text{Sign} \rightarrow \text{Spectral} \rightarrow \text{Sign}$) can safely be applied to all layers as done in all subsequent experiments.

\paragraph{SODA}
The implementation of \ref{eq:SODA} makes the choice $h_k(x) = \iota_{\mathcal{B}(z^0,\gamma_k)}(x)$ where $\mathcal{B}(z,\gamma)=\{x \in \mathcal X \mid \|x-z\|\leq \gamma\}$ which reduces the update to
\begin{equation}\label{eq:SODA:implementation}
\begin{split}
m^{k+1} &= (1-\alpha_k) m^k + \alpha_k \nabla f(y^k,\xi_k) \\
\bar{m}^{k+1} &= (1-\bar{\alpha}_k) m^{k+1} + \bar{\alpha}_k \nabla f(y^k, \xi_k) \\
z^{k+1} & \in z^0 + \gamma_k \lmo_{\mathcal B(0,1)}(\bar{m}^{k+1}) \\ 
x^{k+1} &= (1-\lambda_k)x^k + \lambda_k z^{k+1} \\
y^{k+1} &= (1 - \bar\lambda_k) x^{k+1}  + \bar\lambda_k z^{k+1} 
\end{split}
\end{equation}
with $\alpha_k, \bar\alpha_k, \bar\lambda_k \in [0,1]$, $\lambda_k=1/(k+2)$ and $\gamma_k = \eta_k(k+2)$ similar to \Cref{alg:SODA}.
We initialize $m^0=\nabla f(x^0,\xi_0)$ and $x^0=y^0=z^0$.

For simplicity, consider $\bar \alpha_k=0$ and $\bar\lambda_k=0$, which disable optimism and enables primal extrapolation.
In the context of Frank-Wolfe based methods such as Scion \citep{pethick2025trainingdeeplearningmodels}, SODA then corresponds to i) centering the update around the initialization $z^0$ instead of the origin, and ii) setting the Frank-Wolfe stepsize $\lambda_k=1/(k+2)$ while the radius $\gamma_k=\eta_k (k+2)$ is linearly increasing.

We initialize the momentum as $m^0=\nabla f(x^0,\xi_0)$ following theory, which corresponds to using \texttt{torch.clone(p.grad)}.
We observe that this performs slightly better than the typical $m^0=0$, i.e. \texttt{torch.zeros\_like(p.grad)}, when it comes to SODA.

\paragraph{Setup}
We rely on the tuned baselines for \href{https://github.com/KellerJordan/modded-nanogpt/blob/master/records/track_1_short/2024-10-29_Optimizers/95a9fd44-7c13-49c7-b324-3e7d9e23a499.txt}{AdamW} and \href{https://github.com/KellerJordan/modded-nanogpt/blob/master/records/track_1_short/2024-10-17_DistributedMuon/22d24867-eb5a-4fcc-ae2c-263d0277dfd1.txt}{Muon} from modded-nanogpt \citet{modded_nanogpt_2024} used in \citet{pethick2025trainingdeeplearningmodels}.
For the baselines weight decay refers to setting $\mu_k = \mu \gamma_k$ in \eqref{eq:wd} for weight decay parameter $\mu \geq 0$ and (possibly scheduled) stepsize $\gamma_k$, as is the default in e.g., PyTorch.
Training a 124M model on FineWeb100 for 5k steps (i.e., $1\times$ Chinchilla) takes roughly 20-30min on 4 H200 NVIDIA GPUs.

\paragraph{Varied horizon}
In \Cref{fig:scale_width_horizon}, the learning rate of the baseline is tuned in additional to the weight decay parameter in a grid of powers of $2$ to ensure optimality.
We find that the optimal learning rate is the same across $1\times$ Chinchilla and $4\times$ Chinchilla.

\paragraph{Optimism} \Cref{tab:optimism} explores the role of optimism showing results for different choices of $\alpha_k, \bar\alpha_k$ in \ref{eq:SODA}.
For the remaining hyperparameters we use the same default choices as for the SODA Wrapper, namely $\lambda_k=1/(k+2)$ (uniform averaging) and $\bar\lambda_k=0$ (primal extrapolation).

\begin{table}[ht]
    \centering
    \caption{NanoGPT Validation loss comparing wrapping hidden layers vs. wrapping all layers.}
    \resizebox{\textwidth}{!}{
        \begin{tabular}{l|ccc|cccc}
            \toprule
            Steps & Muon & SODA(hidden) & SODA(all) & Scion & uScion & SODA(hidden) & SODA(all) \\
            \midrule
            5k & 3.28146 & 3.26654 & 3.27493 & 3.28035 & 3.28508 & 3.27534 & 3.27582 \\
            20k & 3.14140 & 3.12621 & 3.12933 & 3.14776 & 3.14746 & 3.13622 & 3.13196 \\
            \bottomrule
            \end{tabular}
    }
        \label{tab:wrapper_layers}
\end{table}
\end{toappendix}

\begin{toappendix}

\begin{table}[H]
    \centering
    \caption{NanoGPT validation loss comparing different values of the dual momentum parameters $\alpha_k, \bar\alpha_k$ for \ref{eq:SODA}.}
    \begin{tabular}{c|cc}
    \toprule
        & 5k steps & 20k steps \\
        \midrule
        $\alpha_k=0.05, \bar\alpha_k=0$ & 3.27613 & 3.13250 \\
        $\alpha_k=0.05, \bar\alpha_k=0.05$ & \textbf{3.26596} & \textbf{3.12376} \\
        $\alpha_k=0.05, \bar\alpha_k=0.1$ & 3.27523 & 3.12543 \\
        $\alpha_k=0.1, \bar\alpha_k=0$ & 3.27582 & 3.13192  \\
        $\alpha_k=0.1, \bar\alpha_k=0.05$ & 3.27171 & 3.12981 \\
        $\alpha_k=0.1, \bar\alpha_k=0.1$ & 3.27903 & 3.13087 \\
    \bottomrule
    \end{tabular}
    \label{tab:optimism}
\end{table}

\begin{figure}[t]
    \centering
    \includegraphics[width=0.5\linewidth]{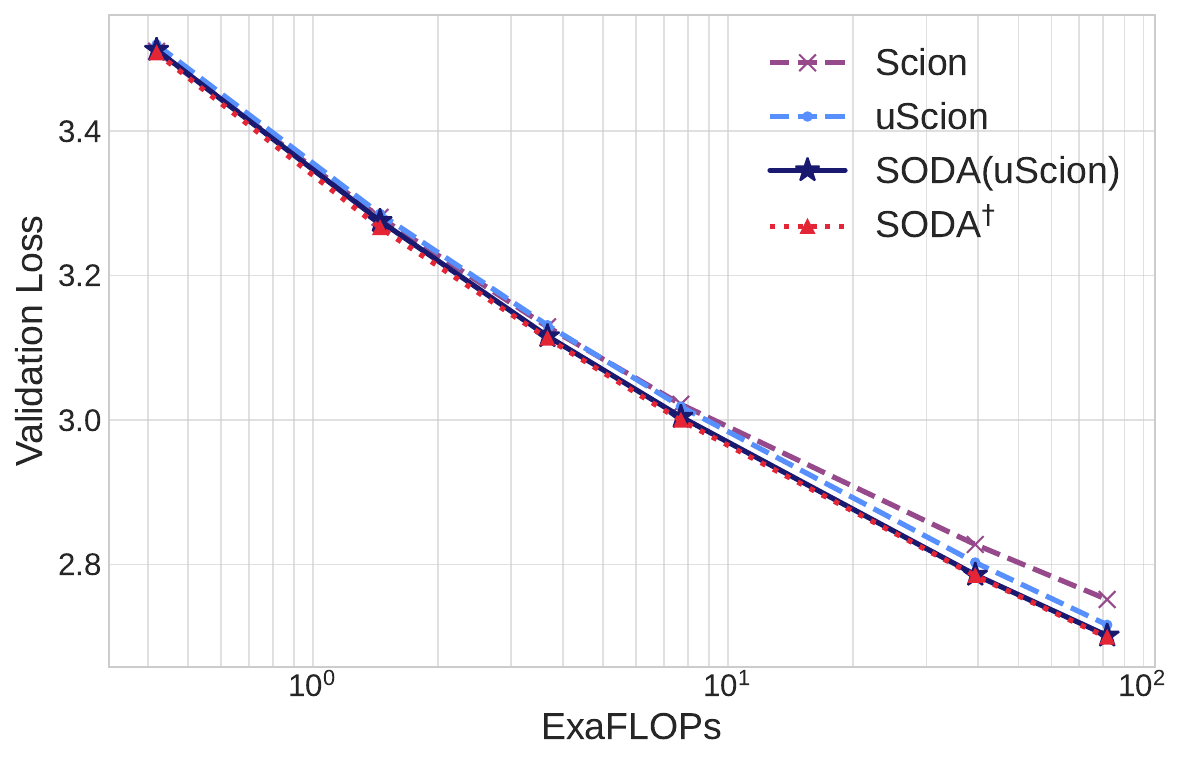}
    \caption{Both SODA$\dag$ and SODA(uScion) are effective under 1$\times$ Chinchilla scaling and the benefit increases with scale.}
    \label{fig:scale_width_horizon_supp}
\end{figure}

\begin{figure}[t]
    \centering
    \includegraphics[width=0.5\linewidth]{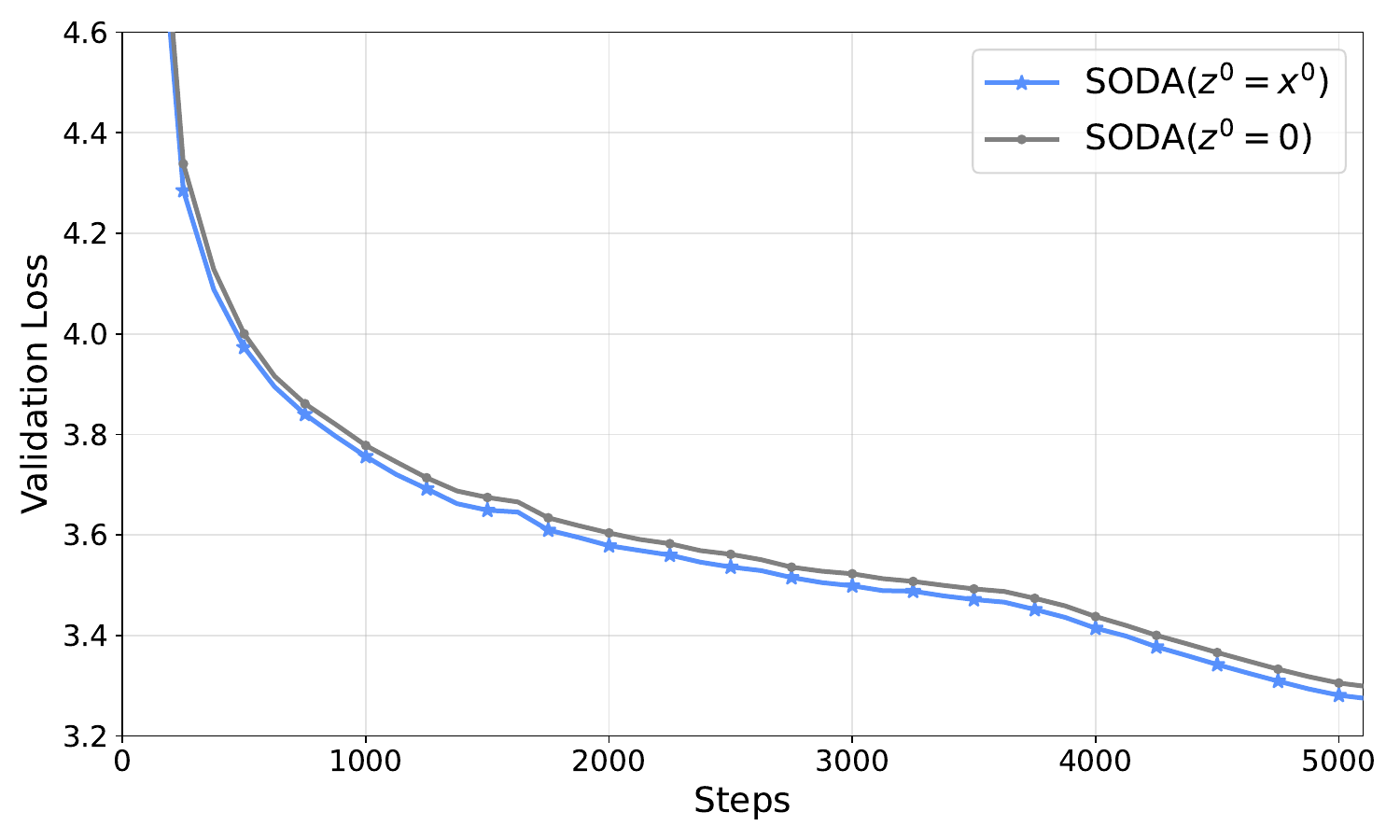}
    \caption{To illustrate the importance of the reference point $z^0$ being the initialization, we instantiate SODA with $z^0=0$, where we observe deterioration in performance.
    This is consistent with theory, since the rate in e.g., \Cref{cor:soda-nonacc-general} depend on the distance to the anchor point, $z^0$, rather than the initial point where the gradient is evaluated, $x^0$.}
    \label{fig:val_loss_z0}
\end{figure}

\begin{table}[!h]
\centering
\caption{For NanoGPT we use the optimized hyperparameters for the 124M model from \citep{pethick2025trainingdeeplearningmodels}.
We use FineWeb100 to scale to the overtrained regime while avoiding data repetition.
For ease of comparison Muon's momentum parameter $\beta$ is written in terms of $\alpha=\bar\alpha$ (where $\beta=1-\alpha$).
SODA$^\dag$ uses $\bar\lambda=0$ and $\lambda_k=1/(k+2)$ similar to the SODA Wrapper.
Both SODA$^\dag$ and Scion use the same norm choice, $\text{Sign} \rightarrow \text{Spectral} \rightarrow \text{Sign}$ (\textit{c.f.} \Cref{tbl:parameter:lmo:same-norm}).}
\label{tbl:hyperparams:nanoGPT}

\begin{tabular}{|c|c|c|c|c|c|}
\hline 
Hyperparameter & AdamW & Muon$\textcolor{blue}{^1}$ & uScion & Scion & SODA$^\dag$ \\
\hline \hline 

Layers & \multicolumn{5}{c|}{12} \\
Head dim & \multicolumn{5}{c|}{128} \\
\hline

Activation function & \multicolumn{2}{c|}{ReLU$^2$} & \multicolumn{3}{c|}{2 $\cdot$ ReLU$^2$} \\
\hline

Vocabulary size & \multicolumn{5}{c|}{50304} \\
Dataset & \multicolumn{5}{c|}{FineWeb100} \\
Batch size & \multicolumn{5}{c|}{512} \\
Block size & \multicolumn{5}{c|}{1024} \\
Iterations $n$ & \multicolumn{5}{c|}{5100 ($1 \times$ Chinchilla)} \\
Warmdown $m$ & \multicolumn{5}{c|}{28.5\% of $n$} \\

Stepsize schedule & \multicolumn{5}{c|}{
$\gamma_k =
\begin{cases}
\gamma & \text{if } k < n-m \\
\gamma \cdot \left(\frac{n - k}{m}\right) & \text{if } k \geq n-m
\end{cases}$
} \\
\hline 

Warmup & 5\% & \multicolumn{4}{c|}{0} \\
\hline 
Adam momentum $\beta_1$ / $\beta_2$ & 0.9 / 0.95 & 0.9 / 0.95 & \multicolumn{2}{c|}{-} & - \\

Averaging parameter $\alpha$ & - & 0.05 & \multicolumn{2}{c|}{0.1} & 0.05 \\
Optimism parameter $\bar\alpha$ & - & 0.05 & \multicolumn{2}{c|}{-} & 0.05 \\
\hline 

Radius $\rho_1$ / $\rho_\ell$ / $\rho_L$ & \multicolumn{2}{c|}{-} & \multicolumn{3}{c|}{- / 50 / 3000} \\
\hline

lr & $2^{-9}$ & $2^{-11}$ & $2^{-12}$ & $2^{-12}$ & $2^{-12}$ \\
\hline

\end{tabular}

$\textcolor{blue}{^1}$ Muon uses Adam for the first/last layer. 
The stepsize for the first/last layer is multiplied by $10$.

\end{table}

\end{toappendix}

\vspace{-0.5em}
\section{Conclusion}\label{sec:conclusion}
\vspace{-0.5em}
This work provides a new perspective on weight decay: beyond acting as a
regularizer, it can be understood as a form of primal averaging.
We show that carefully scheduling this averaging parameter ($\lambda_k$ in
\ref{eq:SODA}) yields acceleration in a precise theoretical sense and can lead to
practical speedups without hyperparameter tuning. 
\looseness=-1

Several directions open up.
One is applying \ref{eq:SODA} to finetuning, where standard weight decay is often
suboptimal. Since \ref{eq:SODA} regularizes with respect to the pretrained model
$z^0$ rather than the origin, it may be better suited for this setting.
Another promising direction is to develop practical accelerated instantiations
of \ref{eq:SODA}.

\section{Acknowledgments}\label{sec:acknowledgments}
This work was funded  by the Swiss National Science Foundation (SNSF) under grant number 2000-1-240094.
This work was supported by the Swiss AI Initiative (2025 Fellowship Program).
This work was supported with project ID \#37 as part of the Swiss AI Initiative, through a grant from the ETH Domain and computational resources provided by the Swiss National Supercomputing Centre (CSCS) under the Alps infrastructure.

\bibliographystyle{plainnat}
\bibliography{refs2.bib,lions-master.bib}

\begin{thebibliography}{52}
\providecommand{\natexlab}[1]{#1}
\providecommand{\url}[1]{\texttt{#1}}
\expandafter\ifx\csname urlstyle\endcsname\relax
  \providecommand{\doi}[1]{doi: #1}\else
  \providecommand{\doi}{doi: \begingroup \urlstyle{rm}\Url}\fi

\bibitem[Bauschke and Lucet(2012)]{bauschke2012fenchel}
H~Bauschke and Yves Lucet.
\newblock What is a {Fenchel} conjugate.
\newblock \emph{Notices of the AMS}, 59\penalty0 (1):\penalty0 44--46, 2012.

\bibitem[Bernstein and Newhouse(2024)]{bernstein2024old}
Jeremy Bernstein and Laker Newhouse.
\newblock Old optimizer, new norm: {An} anthology.
\newblock \emph{arXiv:2409.20325}, 2024.

\bibitem[Bernstein et~al.(2018)Bernstein, Wang, Azizzadenesheli, and Anandkumar]{bernstein2018signsgd}
Jeremy Bernstein, Yu-Xiang Wang, Kamyar Azizzadenesheli, and Animashree Anandkumar.
\newblock {signSGD}: Compressed optimisation for non-convex problems.
\newblock In \emph{International Conference on Machine Learning}, pages 560--569. PMLR, 2018.

\bibitem[Carlson et~al.(2015{\natexlab{a}})Carlson, Cevher, and Carin]{carlson2015stochastic}
David Carlson, Volkan Cevher, and Lawrence Carin.
\newblock Stochastic spectral descent for restricted boltzmann machines.
\newblock In \emph{Artificial Intelligence and Statistics}, 2015{\natexlab{a}}.

\bibitem[Carlson et~al.(2016)Carlson, Hsieh, Collins, Carin, and Cevher]{carlson2016stochastic}
David Carlson, Ya-Ping Hsieh, Edo Collins, Lawrence Carin, and Volkan Cevher.
\newblock Stochastic spectral descent for discrete graphical models.
\newblock \emph{IEEE Journal of Selected Topics in Signal Processing}, 2016.

\bibitem[Carlson et~al.(2015{\natexlab{b}})Carlson, Collins, Hsieh, Carin, and Cevher]{carlson2015preconditioned}
David~E Carlson, Edo Collins, Ya-Ping Hsieh, Lawrence Carin, and Volkan Cevher.
\newblock Preconditioned spectral descent for deep learning.
\newblock In \emph{Proceedings of the 28th International Conference on Neural Information Processing Systems}, pages 2971--2979, 2015{\natexlab{b}}.

\bibitem[Chen et~al.(2024)Chen, Liu, Liang, and Liu]{chen2023lion}
Lizhang Chen, Bo~Liu, Kaizhao Liang, and Qiang Liu.
\newblock Lion secretly solves a constrained optimization: As lyapunov predicts.
\newblock In \emph{International Conference on Learning Representations}, 2024.
\newblock URL \url{https://openreview.net/forum?id=e4xS9ZarDr}.

\bibitem[Chen et~al.(2023)Chen, Liang, Huang, Real, Wang, Pham, Dong, Luong, Hsieh, Lu, and Le]{chen2023symbolic}
Xiangning Chen, Chen Liang, Da~Huang, Esteban Real, Kaiyuan Wang, Hieu Pham, Xuanyi Dong, Thang Luong, Cho-Jui Hsieh, Yifeng Lu, and Quoc~V Le.
\newblock Symbolic discovery of optimization algorithms.
\newblock In \emph{Thirty-seventh Conference on Neural Information Processing Systems}, 2023.
\newblock URL \url{https://openreview.net/forum?id=ne6zeqLFCZ}.

\bibitem[Clarkson(2010)]{ken-fw}
Kenneth~L. Clarkson.
\newblock Coresets, sparse greedy approximation, and the {Frank-Wolfe} algorithm.
\newblock \emph{ACM Trans. Algorithms}, 2010.

\bibitem[Crawshaw et~al.(2025)Crawshaw, Modi, Liu, and Gower]{crawshaw2025exploration}
Michael Crawshaw, Chirag Modi, Mingrui Liu, and Robert~M Gower.
\newblock An exploration of non-euclidean gradient descent: Muon and its many variants.
\newblock \emph{arXiv preprint arXiv:2510.09827}, 2025.

\bibitem[Cutkosky(2019)]{cutkosky2019anytime}
Ashok Cutkosky.
\newblock Anytime online-to-batch, optimism and acceleration.
\newblock In \emph{International conference on machine learning}, pages 1446--1454. PMLR, 2019.

\bibitem[Defazio et~al.(2024)Defazio, Yang, Mehta, Mishchenko, Khaled, and Cutkosky]{defazio2024road}
Aaron Defazio, Xingyu Yang, Harsh Mehta, Konstantin Mishchenko, Ahmed Khaled, and Ashok Cutkosky.
\newblock The road less scheduled.
\newblock In \emph{Advances in Neural Information Processing Systems}, volume~37, pages 9974--10007, 2024.
\newblock \doi{10.52202/079017-0320}.
\newblock URL \url{https://proceedings.neurips.cc/paper_files/paper/2024/file/136b9a13861308c8948cd308ccd02658-Paper-Conference.pdf}.

\bibitem[Defazio et~al.(2025)Defazio, Mishchenko, Raman, Shi, and Xiao]{defazio2025smoothing}
Aaron Defazio, Konstantin Mishchenko, Parameswaran Raman, Hao-Jun~Michael Shi, and Lin Xiao.
\newblock Smoothing {DiLoCo} with primal averaging for faster training of {LLM}s.
\newblock \emph{arXiv preprint arXiv:2512.17131}, 2025.

\bibitem[Douillard et~al.(2023)Douillard, Feng, Rusu, Chhaparia, Donchev, Kuncoro, Ranzato, Szlam, and Shen]{douillard2023diloco}
Arthur Douillard, Qixuan Feng, Andrei~A Rusu, Rachita Chhaparia, Yani Donchev, Adhiguna Kuncoro, Marc'Aurelio Ranzato, Arthur Szlam, and Jiajun Shen.
\newblock Diloco: Distributed low-communication training of language models.
\newblock \emph{arXiv preprint arXiv:2311.08105}, 2023.

\bibitem[Dozat(2016)]{dozat2016incorporating}
Timothy Dozat.
\newblock Incorporating {Nesterov} momentum into {Adam}.
\newblock 2016.

\bibitem[Ferbach et~al.(2026)Ferbach, Paquette, Gidel, Everett, and Paquette]{ferbach2026logarithmic}
Damien Ferbach, Courtney Paquette, Gauthier Gidel, Katie Everett, and Elliot Paquette.
\newblock Logarithmic-time schedules for scaling language models with momentum.
\newblock \emph{arXiv preprint arXiv:2602.05298}, 2026.

\bibitem[Frank and Wolfe(1956)]{frank1956algorithm}
Marguerite Frank and Philip Wolfe.
\newblock An algorithm for quadratic programming.
\newblock \emph{Naval research logistics quarterly}, 1956.

\bibitem[Gupta et~al.(2018)Gupta, Koren, and Singer]{Gupta2018ShampooPS}
Vineet Gupta, Tomer Koren, and Yoram Singer.
\newblock Shampoo: Preconditioned stochastic tensor optimization.
\newblock In \emph{International Conference on Machine Learning}, 2018.

\bibitem[Hanson and Pratt(1988)]{hanson1988comparing}
Stephen Hanson and Lorien Pratt.
\newblock Comparing biases for minimal network construction with back-propagation.
\newblock \emph{Advances in neural information processing systems}, 1, 1988.

\bibitem[Hoffmann et~al.(2022)Hoffmann, Borgeaud, Mensch, Buchatskaya, Cai, Rutherford, de~Las~Casas, Hendricks, Welbl, Clark, et~al.]{hoffmann2022training}
Jordan Hoffmann, Sebastian Borgeaud, Arthur Mensch, Elena Buchatskaya, Trevor Cai, Eliza Rutherford, Diego de~Las~Casas, Lisa~Anne Hendricks, Johannes Welbl, Aidan Clark, et~al.
\newblock Training compute-optimal large language models.
\newblock In \emph{Advances in Neural Information Processing Systems}, 2022.

\bibitem[Jaggi(2013)]{jaggi2013revisiting}
Martin Jaggi.
\newblock Revisiting {Frank-Wolfe}: Projection-free sparse convex optimization.
\newblock In \emph{International Conference on Machine Learning}, 2013.

\bibitem[Jelassi and Defazio(2020)]{jelassi2020dual}
Samy Jelassi and Aaron Defazio.
\newblock Dual averaging is surprisingly effective for deep learning optimization.
\newblock \emph{arXiv preprint arXiv:2010.10502}, 2020.

\bibitem[Jordan et~al.(2024{\natexlab{a}})Jordan, Bernstein, Rappazzo, @fernbear.bsky.social, Vlado, Jiacheng, Cesista, Koszarsky, and @Grad62304977]{modded_nanogpt_2024}
Keller Jordan, Jeremy Bernstein, Brendan Rappazzo, @fernbear.bsky.social, Boza Vlado, You Jiacheng, Franz Cesista, Braden Koszarsky, and @Grad62304977.
\newblock modded-nanogpt: Speedrunning the nanogpt baseline, 2024{\natexlab{a}}.
\newblock URL \url{https://github.com/KellerJordan/modded-nanogpt}.

\bibitem[Jordan et~al.(2024{\natexlab{b}})Jordan, Jin, Boza, Jiacheng, Cecista, Newhouse, and Bernstein]{jordan2024muon}
Keller Jordan, Yuchen Jin, Vlado Boza, You Jiacheng, Franz Cecista, Laker Newhouse, and Jeremy Bernstein.
\newblock Muon: An optimizer for hidden layers in neural networks, 2024{\natexlab{b}}.

\bibitem[Joulani et~al.(2020)Joulani, Raj, Gyorgy, and Szepesv{\'a}ri]{joulani2020simpler}
Pooria Joulani, Anant Raj, Andras Gyorgy, and Csaba Szepesv{\'a}ri.
\newblock A simpler approach to accelerated optimization: iterative averaging meets optimism.
\newblock In \emph{International conference on machine learning}, pages 4984--4993. PMLR, 2020.

\bibitem[Kallusky et~al.(2025)Kallusky, Rao, Nandavanam, and Shi]{kallusky2025snoo}
Dominik Kallusky, Vinay Rao, Vishal Nandavanam, and Hao-Jun~Michael Shi.
\newblock {SNOO}: Step-k {Nesterov} outer optimizer-the surprising effectiveness of {Nesterov} momentum applied to pseudo-gradients.
\newblock \emph{arXiv preprint arXiv:2510.15830}, 2025.

\bibitem[Kavis et~al.(2019)Kavis, Levy, Bach, and Cevher]{kavis2019unixgrad}
Ali Kavis, Kfir~Y Levy, Francis Bach, and Volkan Cevher.
\newblock Unixgrad: A universal, adaptive algorithm with optimal guarantees for constrained optimization.
\newblock \emph{Advances In Neural Information Processing Systems 32 (Nips 2019)}, 32\penalty0 (CONF), 2019.

\bibitem[Kingma and Ba(2014)]{kingma2014adam}
DP~Kingma and Jimmy Ba.
\newblock {Adam}: A method for stochastic optimization.
\newblock In \emph{International Conference on Learning Representations}, 2014.

\bibitem[Kunstner et~al.(2023)Kunstner, Chen, Lavington, and Schmidt]{kunstner2023noise}
Frederik Kunstner, Jacques Chen, Jonathan~Wilder Lavington, and Mark Schmidt.
\newblock Noise is not the main factor behind the gap between sgd and adam on transformers, but sign descent might be.
\newblock \emph{arXiv preprint arXiv:2304.13960}, 2023.

\bibitem[Lan(2012)]{lan2012optimal}
Guanghui Lan.
\newblock An optimal method for stochastic composite optimization.
\newblock \emph{Mathematical Programming}, 133\penalty0 (1):\penalty0 365--397, 2012.

\bibitem[Loshchilov and Hutter(2019)]{loshchilov2017decoupled}
Ilya Loshchilov and Frank Hutter.
\newblock Decoupled weight decay regularization.
\newblock In \emph{International Conference on Learning Representations}, 2019.

\bibitem[Mohri and Yang(2016)]{mohri2016accelerating}
Mehryar Mohri and Scott Yang.
\newblock Accelerating online convex optimization via adaptive prediction.
\newblock In \emph{Artificial Intelligence and Statistics}, pages 848--856. PMLR, 2016.

\bibitem[Mokhtari et~al.(2020)Mokhtari, Hassani, and Karbasi]{mokhtari2020stochastic}
Aryan Mokhtari, Hamed Hassani, and Amin Karbasi.
\newblock Stochastic conditional gradient methods: From convex minimization to submodular maximization.
\newblock \emph{Journal of Machine Learning Research}, 2020.

\bibitem[Morwani et~al.(2025)Morwani, Vyas, Zhang, and Kakade]{morwani2025connections}
Depen Morwani, Nikhil Vyas, Hanlin Zhang, and Sham Kakade.
\newblock Connections between schedule-free optimizers, {AdEMAMix}, and accelerated sgd variants.
\newblock \emph{arXiv preprint arXiv:2502.02431}, 2025.

\bibitem[Nesterov(2005)]{nesterov2005smooth}
Yu~Nesterov.
\newblock Smooth minimization of non-smooth functions.
\newblock \emph{Mathematical programming}, 103:\penalty0 127--152, 2005.

\bibitem[Nesterov(2012)]{nesterov2012efficiency}
Yu~Nesterov.
\newblock Efficiency of coordinate descent methods on huge-scale optimization problems.
\newblock \emph{SIAM Journal on Optimization}, 22\penalty0 (2):\penalty0 341--362, 2012.

\bibitem[Nesterov and Shikhman(2015)]{nesterov2015quasi}
Yu~Nesterov and Vladimir Shikhman.
\newblock Quasi-monotone subgradient methods for nonsmooth convex minimization.
\newblock \emph{Journal of Optimization Theory and Applications}, 165\penalty0 (3):\penalty0 917--940, 2015.

\bibitem[Nesterov(2009)]{nesterov2009primal}
Yurii Nesterov.
\newblock Primal-dual subgradient methods for convex problems.
\newblock \emph{Mathematical programming}, 120\penalty0 (1):\penalty0 221--259, 2009.

\bibitem[Orabona(2019)]{Orabona2019}
Francesco Orabona.
\newblock A modern introduction to online learning.
\newblock \emph{CoRR}, abs/1912.13213, 2019.
\newblock URL \url{http://arxiv.org/abs/1912.13213}.

\bibitem[Pagliardini et~al.(2024)Pagliardini, Ablin, and Grangier]{ademamix2024}
Matteo Pagliardini, Pierre Ablin, and David Grangier.
\newblock The {AdEMAMix} optimizer: Better, faster, older.
\newblock \emph{arXiv preprint arXiv:2409.03137}, 2024.

\bibitem[Pethick et~al.(2025{\natexlab{a}})Pethick, Xie, Antonakopoulos, Zhu, Silveti-Falls, and Cevher]{pethick2025trainingdeeplearningmodels}
Thomas Pethick, Wanyun Xie, Kimon Antonakopoulos, Zhenyu Zhu, Antonio Silveti-Falls, and Volkan Cevher.
\newblock Training deep learning models with norm-constrained {LMOs}.
\newblock In \emph{International Conference on Machine Learning}, 2025{\natexlab{a}}.

\bibitem[Pethick et~al.(2025{\natexlab{b}})Pethick, Xie, Erdogan, Antonakopoulos, Silveti-Falls, and Cevher]{pethick2025generalized}
Thomas Pethick, Wanyun Xie, Mete Erdogan, Kimon Antonakopoulos, Tony Silveti-Falls, and Volkan Cevher.
\newblock Generalized gradient norm clipping \& non-euclidean {$(L_0, L_1)$}-smoothness.
\newblock \emph{arXiv preprint arXiv:2506.01913}, 2025{\natexlab{b}}.

\bibitem[Qiu et~al.(2025)Qiu, Chen, Phan, Lei, and Wilson]{qiu2025hyperparameter}
Shikai Qiu, Zixi Chen, Hoang Phan, Qi~Lei, and Andrew~Gordon Wilson.
\newblock Hyperparameter transfer enables consistent gains of matrix-preconditioned optimizers across scales.
\newblock \emph{arXiv preprint arXiv:2512.05620}, 2025.

\bibitem[Rakhlin and Sridharan(2013)]{rakhlin2013online}
Alexander Rakhlin and Karthik Sridharan.
\newblock Online learning with predictable sequences.
\newblock In \emph{Conference on Learning Theory}, pages 993--1019. PMLR, 2013.

\bibitem[Scetbon et~al.(2025)Scetbon, Ma, Gong, and Meeds]{scetbon2025gradient}
Meyer Scetbon, Chao Ma, Wenbo Gong, and Edward Meeds.
\newblock Gradient multi-normalization for stateless and scalable {LLM} training.
\newblock \emph{arXiv preprint arXiv:2502.06742}, 2025.

\bibitem[Schaipp et~al.(2025)Schaipp, H{\"a}gele, Taylor, Simsekli, and Bach]{schaipp2025surprising}
Fabian Schaipp, Alexander H{\"a}gele, Adrien Taylor, Umut Simsekli, and Francis Bach.
\newblock The surprising agreement between convex optimization theory and learning-rate scheduling for large model training.
\newblock \emph{arXiv preprint arXiv:2501.18965}, 2025.

\bibitem[Tseng(2008)]{tseng2008accelerated}
Paul Tseng.
\newblock On accelerated proximal gradient methods for convex-concave optimization.
\newblock \emph{submitted to SIAM Journal on Optimization}, 2\penalty0 (3), 2008.

\bibitem[Xiao(2024)]{xiao2024rethinking}
Lechao Xiao.
\newblock Rethinking conventional wisdom in machine learning: From generalization to scaling.
\newblock \emph{arXiv preprint arXiv:2409.15156}, 2024.

\bibitem[Xie and Li(2024)]{xie2024implicit}
Shuo Xie and Zhiyuan Li.
\newblock Implicit bias of {AdamW}: $\ell_\infty$ norm constrained optimization.
\newblock \emph{arXiv preprint arXiv:2404.04454}, 2024.

\bibitem[Xie et~al.(2025)Xie, Wei, Cao, Zhao, Deng, Li, Dai, Gao, Chang, Zhao, et~al.]{xie2025mhc}
Zhenda Xie, Yixuan Wei, Huanqi Cao, Chenggang Zhao, Chengqi Deng, Jiashi Li, Damai Dai, Huazuo Gao, Jiang Chang, Liang Zhao, et~al.
\newblock {mHC}: Manifold-constrained hyper-connections.
\newblock \emph{arXiv preprint arXiv:2512.24880}, 2025.

\bibitem[Yang et~al.(2021)Yang, Hu, Babuschkin, Sidor, Liu, Farhi, Ryder, Pachocki, Chen, and Gao]{yang2022tensor}
Ge~Yang, Edward Hu, Igor Babuschkin, Szymon Sidor, Xiaodong Liu, David Farhi, Nick Ryder, Jakub Pachocki, Weizhu Chen, and Jianfeng Gao.
\newblock Tuning large neural networks via zero-shot hyperparameter transfer.
\newblock In \emph{Advances in Neural Information Processing Systems}, 2021.

\bibitem[Zhang et~al.(2019)Zhang, Lucas, Hinton, and Ba]{zhangLookaheadOptimizerSteps2019}
Michael~R. Zhang, James Lucas, Geoffrey Hinton, and Jimmy Ba.
\newblock Lookahead {{Optimizer}}: K steps forward, 1 step back.
\newblock \emph{arXiv}, December 2019.

\end{thebibliography}

\newpage
\appendix
\onecolumn

\begin{center}
\vspace{7pt}
{\Large \fontseries{bx}\selectfont Appendix}
\end{center}

\renewcommand{\contentsname}{Table of Contents}
\etocdepthtag.toc{mtappendix}
\etocsettagdepth{mtchapter}{none}
\etocsettagdepth{mtappendix}{section}
\tableofcontents

\clearpage

\nosectionappendix

\end{document}